\documentclass[10pt]{article}

\usepackage{amsmath,amsthm,amssymb,amsfonts}

\usepackage[a4paper]{geometry}

\usepackage{a4wide}

\usepackage{graphicx}

\usepackage{subfigure}

\usepackage{url}

\usepackage{hyperref}

\usepackage{color}
\usepackage[dvipsnames]{xcolor}

\usepackage[margin=10pt,font=small,labelfont=bf, labelsep=endash]{caption}

\usepackage{authblk}

\usepackage{framed}

\usepackage{caption}

 \newcommand{\ATLnote}[1]{}
\newcommand{\AWnote}[1]{}

\newtheorem{theorem}{Theorem}
\newtheorem{corollary}[theorem]{Corollary}
\newtheorem{lemma}[theorem]{Lemma}

\DeclareMathOperator*{\argmin}{arg\,min}

\title{Parameter Inference of Time Series by Delay Embeddings and Learning Differentiable Operators}
\author[1]{Alex Tong Lin}
\author[2]{Adrian S. Wong}
\author[3]{Robert Martin}
\author[4]{Stanley J. Osher}
\author[5]{Daniel Eckhardt}

\affil[1, 4]{University of California, Los Angeles}
\affil[2]{Jacobs Technology Inc. Edwards AFB}
\affil[3]{DEVCOM-ARL U.S.Army Research Office}
\affil[5]{Air Force Research Laboratory, Edwards AFB}
\date{}
\begin{document}
	
\maketitle

\begin{abstract}
    We provide a method to identify system parameters of dynamical systems, called ID-ODE -- Inference by Differentiation and Observing Delay Embeddings. In this setting, we are given a dataset of trajectories from a dynamical system with system parameter labels. Our goal is to identify system parameters of new trajectories. The given trajectories may or may not encompass the full state of the system, and we may only observe a one-dimensional time series. In the latter case, we reconstruct the full state by using delay embeddings, and under sufficient conditions, Taken's Embedding Theorem assures us the reconstruction is diffeomorphic to the original. This allows our method to work on time series. Our method works by first learning the velocity operator (as given or reconstructed) with a neural network having both state and system parameters as variable inputs. Then on new trajectories we backpropagate prediction errors to the system parameter inputs giving us a gradient. We then use gradient descent to infer the correct system parameter. We demonstrate the efficacy of our approach on many numerical examples: the Lorenz system, Lorenz96, Lotka-Volterra Predator-Prey, and the Compound Double Pendulum. We also apply our algorithm on a real-world dataset: propulsion of the Hall-effect Thruster (HET).
\end{abstract}

\let\thefootnote\relax\footnotetext{Distribution Statement A: Approved for Public Release; Distribution is Unlimited. PA Clearance AFRL-2022-1216 (Submitted for approval on March 2, 2022)}

\section{Introduction}

	When studying dynamical systems arising from real phenomenon, we often don't have a mathematical description, and instead must settle with observed data. Thus, many researchers and practitioners are using machine learning to ascertain properties of these systems, wherein there is ample data. One approach tries to infer the governing equations, whether represented symbolically or as a neural network. In our approach, we are not necessarily trying to reconstruct these equations, but rather to infer the system parameters, even when these equations are unknown. For example, in tackling the problem of parameter inference from the Lorenz system (which was developed as a simplified model of atmospheric convection \cite{DeterministicNonperiodicFlow}): 
	\begin{equation*}
		\begin{split}
			\frac{dx}{dt} &= \sigma (y - x) \\
			\frac{dy}{dt} &= x (\rho - z) - y \\
			\frac{dz}{dt} &= xy - \beta z
		\end{split}
	\end{equation*} 
	there are three parameters: $\sigma$ is the Prandtl number, $\rho$ is the Rayleigh number, and $\beta$ is related to the physical proportions of the region under consideration \cite{SparrowLorenz}. In our setting, we may only have access to trajectories or time series with system parameter labels, and our goal is to infer the system parameters of new trajectories or time series. 

	In this work, we introduce ID-ODE -- Inference by Differentiation and Observing Delay Embeddings. This method tackles the following scenario: We seek to identify system parameters of unlabeled trajectories, given that we have a dataset of trajectory data with labeled system parameters. These unlabeled trajectories may even have parameter labels not present in the training data, and thus we must interpolate. So the neural network must not only learn the relationship between a dynamical system and the system parameters, but it must learn the dynamical system itself as well. This gets complicated when we only have time series data -- a frequent occurrence in the real world -- and must reconstruct the state space. Our contribution lies in the insight that we can utilize delay embeddings for time series in order to reconstruct the state space giving us a proper ODE. We then use neural networks to learn the velocity operator/right-hand side of the ODE. 
    Intuitively, the velocity operator is usually much simpler to describe, and thus perhaps learn, than the solution of the ODE: the Lorenz system has a right-hand side that is a second-order polynomial, but the trajectories can be quite complex -- we don't even have a closed-form solution. Indeed, a big reason we study ODEs is because it's easier to describe a system this way, than using the solution. Put simply: it's easier to describe the rules than to describe the outcome.
 
    We bring this intuition to time series: we delay embed the time series, and then learn the velocity/right-hand side with a neural network. And because one can backpropagate through the neural network, then we can construct gradients in parameter space in order to perform parameter inference. We show delay embeddings also have benefits even when we have the full state space, as delay embedding the full state space improves performance, acting as a smoothing regularizer. 

    Our paper is organized as follows: We first provide an overview of Related Works (Section \ref{sec:related-works}), from previous approaches to parameter inference, and also discuss related fields. We then dive into the Methods (Section \ref{sec:methods}), where we provide an overview of of dynamical systems and ID-ODE. After, we provide an Analysis (Section \ref{sec:analysis}), where we prove our method converges in the practical setting where the mapping from system parameter to dynamical system is affine, and we only have trajectory points. Finally, in the Experiments (Section \ref{sec:experiments}), we demonstrate the effectiveness of ID-ODE on a number of popular chaotic systems, as well as a real-world example with the Hall-effect Thruster.
	
\section{Related Works}
\label{sec:related-works}

Now that there is an abundance of data, applying machine learning techniques to dynamical systems has become a recent trend. Some researchers have taken to performing system identification by identifying the governing symbolic equations through data \cite{Brunton3932, doi:10.1137/18M1191944, Champion22445}, while others have taken to representing the governing equations through neural networks \cite{661124, raissi2018multistep, doi:10.1137/18M1188227, NEGRINI2021110549, DBLP:journals/corr/abs-2110-08382}. Delay embeddings in terms of system identification have also been examined in \cite{kamb2020time, doi:10.1137/18M1188227}. Although often in system identification, the dynamical parameters are fixed, and thus there is a need for approaches that directly perform parameter inference.

In terms of parameter inference, approaches from an optimal transport view are considered in \cite{yang2021optimal}, where they know the form of the equation. In \cite{doi:10.1063/1.5085780}, they also take an optimal transport view but infer only one parameter. Parameter inference using the architecture of Echo State Networks has also been explored in \cite{alao2021discovering}, where they work on the Lorenz equation.

From the perspective of simulations, using neural networks to infer parameters from a have gone back as far as \cite{grzeszczuk1998neuroanimator} which was applied to the field of animations and took a control theory view. And \cite{de2018end} also constructs a gradient in parameter space like us, but is applied to simulations where the full state space is accessible and without noise.

In the field of Time Series Classification \cite{bagnall2017great}, they seek to classify time series with a label, much like in computer vision where practitioners seek to classify images. Such algorithms like ROCKET \cite{dempster2020rocket}, HIVE-COTE \cite{lines2016hive}, MiniRocket \cite{dempster2021minirocket} have been proposed. Compared to our method, we are solving a different problem, where additionally we seek to identify time series with labels not originally found in the training day, but must be interpolated.

\section{Methods}
\label{sec:methods}

Here we provide an overview of the concepts and techniques that make up ID-ODE. We first present the method of delay embeddings for time series in order to reconstruct the state space of the system. We then give a brief discussion of dynamical systems, and provide common notation. Afterwards we explain the learning phase of the neural network, which seeks to learn the dynamics operator. And finally, we explain the inference stage, where we produce gradients in parameter space in order to infer system parameters.

\subsection{Data preprocessing: Delay Embedding}

\subsubsection{Delay Embedding}

    When dealing with real-world dynamical systems, we often are only given a one-dimensional time series, but to \emph{uniquely} identify the state of the (deterministic) system we require access to the full state space. Failure to fully span the state space makes the system appear non-deterministic, but this is only an artifact of not having the full state. Such issues often arise in experimental time series of nonlinear systems, and are particularly tricky to deal with when the state space and physical model are both unknown. Most, or almost all, measurement of the system is unlikely to be a state variable, and in addition, the number of measured variables is likely less than the dimension of the state space. To further complicate matters, the introduction of noise in the measurements also adds another layer of complexity, which will be discussed toward the end of this subsection.
    
    A clever workaround to all these predicaments is to introduce time delay embedding as a preprocessing step for the data. This embedding technique is possible due to the famous Takens' Embedding Theorem and is widely used in time series analysis, particularly in the presence of chaotic orbits \cite{takens1981detecting}. Takens' Embedding Theorem states that an $n$-dimensional (possibly fractal) manifold of a time series can be reconstructed diffeomorphically by creating a time delay vector of at most $d = 2n$ entries, each with a fixed delay of $\tau$ \cite{sauer1991embedology}. The number of entries may be less than the embedding dimension and this varies between systems\textsuperscript{1}.
    
    \begin{equation*}
        \textbf{y}(t) := (y(t), y(t - \tau), y(t - 2\tau), \ldots, y(t - (d-1) \tau))
    \end{equation*}
    
    \footnotetext{\textsuperscript{1}For example, the Lorenz system with standard parameters has a fractal box-counting dimension of $n=2.06$. Takens' theorem states that we need at most $m = 5 > 4.12$ components in the time delay vector to reconstruct the attractor. As it happens, the Lorenz attractor can be reconstructed with just 3 components in the time delay vector.}   
    
    The time delay embedding is equivalent to incorporating information of higher order derivatives as surrogate state variables. There are additional considerations, such as having $\tau$ practically large enough for the states $y(t)$ and $y(t-\tau)$ may sufficiently decorrelate. Having $\tau$ too small means that the states $y(t)$ and $y(t-\tau)$ are strongly correlated, therefore containing no additional information.
    
\subsubsection{Minimum Embedding}

    The act of using time delay embedding is akin to unfolding the attractor such that \emph{artificially} intersecting trajectories no longer intersect. Such ``artificially intersecting" trajectories are also called ``false neighbors". The most prevalent tests involving the determination of minimum embedding dimension will test, in some form or another, the ratio of ``neighborliness" observed. The Kennel \cite{kennel1992determining, kennel2002false} and Cao \cite{cao1997practical} algorithms both quantify ``neighborliness" by tracking the average distance between a point and its nearest neighbour as the embedding dimension increases. The False First Nearest Neighbor algorithm \cite{krakovska2015use} tracks ``neighborliness" depending on whether nearest neighbors remain nearest neighbors after consecutive embeddings. Both methods are suitable for our application, and because neither of them have a clear advantage, we employ both methods in our determination of minimum embedding.

\subsection{The dynamical system}

	We consider a dynamical system,
		\begin{equation}\label{eqn:dynamicalsystem}
			\frac{d}{dt}\textbf{x}(t, \boldsymbol{\alpha}) = F(\textbf{x}(t, \boldsymbol{\alpha}), \boldsymbol{\alpha}), \quad \textbf{x}_0\in\mathbb{R}^n
		\end{equation}
    with $\textbf{x}_0\in\mathbb{R}^n$, $\boldsymbol{\alpha}\in A \subseteq \mathbb{R}^m$, $\textbf{x}:[0, T], A \rightarrow \mathbb{R}^n$, $F:\mathbb{R}^{n+m} \rightarrow \mathbb{R}^n$, and $t\in [0, T]$. For notational simplicity, we write $\textbf{x}(t) = \textbf{x}(t, \boldsymbol{\alpha})$, where the dependence on $\boldsymbol{\alpha}$ is implicit. For the system in eq. \eqref{eqn:dynamicalsystem}, we assume sufficient regularity conditions (e.g. $F$ is uniformly Lipschitz in $\textbf{x}$ for each $\boldsymbol{\alpha}$), so that existence and uniqueness is guaranteed \cite{coddington1955theory}. 
	
\subsection{The Learning Stage: Training to minimize integration error}
    \label{subsec:learning_stage}
	
	The Learning Stage in our method requires training a neural network to learn the velocity map of the ODE. More precisely, given that our data consists of trajectories and corresponding parameters,
		\begin{equation*}
			\left\{ \left(\{\textbf{x}(t_k, \boldsymbol{\alpha})\}_{k=1}^N, \boldsymbol{\alpha}\right) \right\}_{\boldsymbol{\alpha} \in A}
		\end{equation*}
	then we want to minimize
		\begin{equation}\label{eqn:nnloss}
			\min_\theta \mathbb{E}_{k, \boldsymbol{\alpha}}\left[ \frac{1}{2} \left\|\frac{\textbf{x}(t_{k+1}) - \textbf{x}(t_k)}{\Delta t} - F_\theta(\textbf{x}(t_k), \boldsymbol{\alpha}) \right\|^2 \right]
		\end{equation}
    where the expectation is taken over trajectory points, $\textbf{x}(t_k)$ and $\textbf{x}(t_{k+1})$, and system parameters, $\boldsymbol{\alpha}$. If performance is not sufficient with just forward Euler integration, we have found delay embedding the full state space can improve the method. Namely for $\tau, d > 0$, if we let,
        \begin{equation*}
            \tilde{\textbf{x}}(t_k) = (\textbf{x}(t_k), \textbf{x}(t_{\tau(k-1)}), \ldots, \textbf{x}(t_{\tau(k-(d-1))}))
        \end{equation*}
    then we can minimize the following,
        \begin{equation*}
            \min_\theta \mathbb{E}_{k, \boldsymbol{\alpha}}\left[ \frac{1}{2} \left\|\frac{\tilde{\textbf{x}}(t_{k+1}) - \tilde{\textbf{x}}(t_k)}{\Delta t} - F_\theta(\tilde{\textbf{x}}(t_k), \boldsymbol{\alpha}) \right\|^2 \right]
        \end{equation*}
    We note that if we take intuition from numerical schemes, we might choose $\tau = 1$ (i.e. consecutive time delays inspired by multi-step integration schemes), but we demonstrate empirically that this does not always improve performance. Rather, one needs $\tau$ to be bigger, i.e. spaced out time delays, and we show this improves the method, acting as a smoothing regularizer. These points are discussed and expanded upon in Section \ref{sec:compound-double-pendulum-full}.
	
\subsection{The Inference Stage: Producing gradients in parameter space}
    \label{subsec:inference_stage}
	
	In the Inference Stage, we complete our goal in identifying system parameters of unlabeled trajectories. After training a neural network $F_\theta = F_\theta(\textbf{x}, \boldsymbol{\alpha})$ as in the Learning Stage, we now have a computationally differentiable operator -- $F_\theta$ -- with which we can produce gradients with respect to system parameters. Then given a trajectory $\{\textbf{x}(t_k)\}_{k=1}^N$ whose corresponding system parameters are unknown, this parameter (or an equivalent parameter) must be the argument minimum of the following loss
		\begin{equation*}
			\min_{\boldsymbol{\alpha}} \mathbb{E}_{k}\left[ \frac{1}{2} \left\|\frac{\textbf{x}(t_{k+1}) - \textbf{x}(t_k)}{\Delta t} - F_\theta(\textbf{x}(t_k), \boldsymbol{\alpha}) \right\|^2 \right]
		\end{equation*}
	which we note is similar to eq. \eqref{eqn:nnloss}, with the crucial differences being that we now minimize with respect to $\boldsymbol{\alpha}$ and the expectation is now only in $k$. Now since $F_\theta$ is differentiable, then we have available the following gradient descent update rule:
		\begin{equation*}
			\boldsymbol{\alpha}_{j+1} = \boldsymbol{\alpha}_j - h\, \nabla_{\boldsymbol{\alpha}} \left( \mathbb{E}_{k}\left[ \frac{1}{2} \left\|\frac{\textbf{x}(t_{k+1}) - \textbf{x}(t_k)}{\Delta t} - F_\theta(\textbf{x}(t_k), \boldsymbol{\alpha}) \right\|^2 \right] \right)
		\end{equation*}
	with $h$ the gradient step-size. Of course we can also use other gradient update rules, but we display the simple gradient descent method for simplicity. In our numerical experiments, we take advantage of automatic differentiation to compute gradients.

\section{Analysis}
\label{sec:analysis}

    We provide analysis of the Inference stage of our method from varying perspectives. We will mainly be assuming that the mapping from the parameters to the velocity field, i.e. $\boldsymbol{\alpha} \mapsto F(\textbf{x}, \boldsymbol{\alpha})$ is linear or affine for any fixed $\textbf{x}$, which actually covers a variety of commonly encountered chaotic maps: Lorenz, Lorenz96, Van der Pol oscillator, Lotka-Volterra Preday-Prey, R\"{o}ssler's attractor, Chua's circuit, etc \cite{wiki:List_of_chaotic_maps}. We demonstrate this for a couple systems in Section \ref{sec:linear-and-affine}. Of course, we still need a neural network to learn the \emph{nonlinear} (as a function of $\textbf{x}$ and $\boldsymbol{\alpha}$) dynamics, as we only have access to trajectory points and parameter labels, which is the main motivation for these theorems. 
    
    For theoretical analysis, we change our expectation over time $t$ instead of an expectation over trajectory points with indices $k$. Finally, we note throughout our analysis, $\| \cdot \|$ will either denote a vector norm, or a matrix/operator norm, the context clarifying.

\subsection{Time-step error}
    
    We firstly show convergence as $\Delta t \rightarrow 0$, when we have the true velocity mapping $F$ available.
    
    \begin{theorem}[Linear]
    \label{theorem:dt-convergence}
        Suppose we are given a dynamical system,
            \begin{equation*}
                \frac{d}{dt}\textbf{x}(t, \boldsymbol{\alpha}) = F(\textbf{x}(t, \boldsymbol{\alpha}), \boldsymbol{\alpha}), \quad t\in [0, T + \delta],
            \end{equation*}
        where $F$ is smooth with respect to $\textbf{x}$, and such that $L_{\textbf{x}}:\boldsymbol{\alpha}\mapsto F(\textbf{x}, \boldsymbol{\alpha})$ is a linear map for all $\textbf{x}$.
        
        Let $\Delta t < \delta$, $\mathcal{I} = [0, T]$, and $\{\textbf{x}(t, \boldsymbol{\alpha}_0)\}_{t\in\mathcal{I}}$ be a trajectory generated by the dynamical system such that $L_{\textbf{x}(t, \boldsymbol{\alpha}_0)}$ is uniformly bounded for $t\in \mathcal{I}$, and $\mathbb{E}_{t\in\mathcal{I}}[(L_{\textbf{x}(t, \boldsymbol{\alpha}_0)})^\top L_{\textbf{x}(t, \boldsymbol{\alpha}_0)}]$ has a bounded inverse. Then if
            \begin{equation*}
                \boldsymbol{\alpha}_{\Delta t} = \argmin_{\boldsymbol{\alpha}} \mathbb{E}_{t \in \mathcal{I}} \left[ \left\| \frac{\textbf{x}(t + \Delta t, \boldsymbol{\alpha}_0) - \textbf{x}(t, \boldsymbol{\alpha}_0)}{\Delta t} - F(\textbf{x}(t, \boldsymbol{\alpha}_0), \boldsymbol{\alpha}) \right\|^2 \right]
            \end{equation*}
        then $\boldsymbol{\alpha}_{\Delta t} \rightarrow \boldsymbol{\alpha}_0$ as $\Delta t \rightarrow 0$.

    \end{theorem}

    The proof is found in Section \ref{app:dt-convergence-proof}. The above problem is convex because the mapping $\boldsymbol{\alpha} \mapsto F(\textbf{x}, \boldsymbol{\alpha})$ is linear in $\boldsymbol{\alpha}$. Note the assumption that $\mathbb{E}_{t\in\mathcal{I}}[(L_{\textbf{x}(t, \boldsymbol{\alpha}_0)})^\top L_{\textbf{x}(t, \boldsymbol{\alpha}_0)}]$ has a bounded inverse implies each $\boldsymbol{\alpha}$ generates a unique vector field $F$. Theorem \ref{theorem:dt-convergence} covers the case when $\boldsymbol{\alpha} \rightarrow F(\textbf{x}, \boldsymbol{\alpha})$ is linear. The case when it is instead affine follows easily,

    \begin{corollary}[Affine]
    \label{corollary:dt-convergence-affine}
        Suppose we are given a dynamical system,
            \begin{equation*}
                \frac{d}{dt}\textbf{x}(t, \boldsymbol{\alpha}) = F(\textbf{x}(t, \boldsymbol{\alpha}), \boldsymbol{\alpha}), \quad t\in [0, T + \delta],
            \end{equation*}
        where $F$ is smooth, and such that, $\Delta t < \delta$, $\mathcal{I} = [0, T]$, and $(L_{\textbf{x}} + b_\textbf{x}):\boldsymbol{\alpha}\mapsto F(\textbf{x}, \boldsymbol{\alpha})$ is an affine map (with $L_\textbf{x}$ the linear part) for all $\textbf{x}$. If $\{\textbf{x}(t, \boldsymbol{\alpha}_0)\}_{t\in\mathcal{I}}$ is a trajectory with the same assumptions in Theorem \ref{theorem:dt-convergence}, then $\boldsymbol{\alpha}_{\Delta t} \rightarrow \boldsymbol{\alpha}_0$ as $\Delta t \rightarrow 0$.
    \end{corollary}

    The proof is contained in Section \ref{corollary:dt-convergence-affine-proof}. Thus Theorem \ref{theorem:dt-convergence} and Corollary \ref{corollary:dt-convergence-affine} both show convergence to the true parameters as $\Delta t$ goes to zero. This, of course, has practical applications as data is only received as points on trajectories, and we don't have access to the true velocity. As a side remark, we can view the finite difference approximation to $F$ as integration noise, and thus when we train we are fitting a noisy signal.

\subsection{Velocity approximation error}

    We also show convergence of parameters when the approximation error converges to zero. Namely, suppose we don't have access to the true velocity, but rather an approximation. Yet the trajectories $\{\textbf{x}(t, \boldsymbol{\alpha}_0)\}_t$ we want to label are generated by the true velocity map.  We let $F$ be the true velocity map, and $F_\theta$ be an approximation such that,
        \begin{equation*}
            F_\theta(\textbf{x}, \boldsymbol{\alpha}) = F(\textbf{x}, \boldsymbol{\alpha}) + e(\textbf{x}, \boldsymbol{\alpha})
        \end{equation*}
    where $e$ is the error. If $F$ is smooth and the function class we are approximating over, e.g. a neural network, is smooth then $e$ is also smooth. Then if we assume certain conditions on $e$, we can make a statement about convergence of parameters:

    \begin{theorem}
        \label{theorem:approx-error-alpha}
        Suppose we are given a dynamical system,
            \begin{equation*}
                \frac{d}{dt}\textbf{x}(t, \boldsymbol{\alpha}) = F(\textbf{x}(t, \boldsymbol{\alpha}), \boldsymbol{\alpha}), \quad t\in [0, T+\delta],
            \end{equation*}
        with $F$ smooth and $\boldsymbol{\alpha} \mapsto F(\textbf{x}, \boldsymbol{\alpha})$ affine for every $\textbf{x}$, and for $\boldsymbol{\alpha} \in A$, with $A$ open and bounded. Let $\Delta t < \delta$, $\mathcal{I} = [0, T]$, and $F_\theta$ be a smooth approximation such that,
            \begin{equation*}
                F_\theta(\textbf{x}, \boldsymbol{\alpha}) = F(\textbf{x}, \boldsymbol{\alpha}) + e(\textbf{x}, \boldsymbol{\alpha}).
            \end{equation*}
        for all $t\in\mathcal{I}$ and for all $\boldsymbol{\alpha} \in A$. Let $\{\textbf{x}(t, \boldsymbol{\alpha}_0)\}_{t\in \mathcal{I}}$ be a trajectory with the assumptions of Corollary \ref{corollary:dt-convergence-affine}, and such that
            \begin{equation}\label{eq:sobolev}
                \|e(\textbf{x}(t, \boldsymbol{\alpha}_0), \boldsymbol{\alpha})\|  + \|J_{e(\textbf{x}(t, \boldsymbol{\alpha}_0), \cdot)} (\boldsymbol{\alpha}) \| \le \varepsilon
            \end{equation}
        (where $J_{e(\textbf{x}(t, \boldsymbol{\alpha}_0), \cdot)}$ is the Jacobian of $e$ with respect to $\boldsymbol{\alpha}$) for all $t\in\mathcal{I}$ and for all $\boldsymbol{\alpha}\in A$. If 
            \begin{equation*}
                \alpha_{\Delta t, \varepsilon} = \argmin_{\boldsymbol{\alpha}} \mathbb{E}_{t \in \mathcal{I}} \left[ \left\| \frac{\textbf{x}(t + \Delta t, \boldsymbol{\alpha}_0) - \textbf{x}(t, \boldsymbol{\alpha}_0)}{\Delta t} - F_\theta(\textbf{x}(t, \boldsymbol{\alpha}_0), \boldsymbol{\alpha}) \right\|^2 \right]
            \end{equation*}
        then $\boldsymbol{\alpha}_{\Delta t, \varepsilon} \rightarrow \boldsymbol{\alpha}_{\Delta t}$ as $\epsilon \rightarrow 0$.
    \end{theorem}

    We note the condition eq. \ref{eq:sobolev} is a Sobolev-type condition on the error, wherein we not only require the error to converge to zero, but also its derivative. The proof is found in Section \ref{theorem:approx-error-alpha-proof}.

\subsection{Consistency of the time delay solution}

    In this theoretical analysis of the time delay, we follow the terminology of \cite{CASDAGLI199152}, and we also \emph{remove} the assumption of linearity/affine with respect to $\boldsymbol{\alpha}$; we only require smoothness of the velocity map with respect to the state. Now suppose we don't have access to the full state space $\boldsymbol{x}(t, \boldsymbol{\alpha})$, but only a time series $u(t, \boldsymbol{\alpha}) = h(\textbf{x}(t, \boldsymbol{\alpha}))$, where $h$ is the \emph{measurement} or \emph{observation} function. Further suppose that for a given delay $\tau$ and dimension $d$, 
        \begin{equation}\label{eq:time-series_delay}
            \underline{u}(t, \boldsymbol{\alpha}) = (u(t, \boldsymbol{\alpha}), u(t - \tau, \boldsymbol{\alpha}), u(t - 2\tau, \boldsymbol{\alpha}), \ldots, u(t - (d - 1)\tau, \boldsymbol{\alpha}))
        \end{equation}
    is an embedding with embedding map $\Phi$, i.e. the \emph{delay reconstruction map}
        \begin{equation}\label{eq:delay-reconstruction-map}
            \Phi(\textbf{x}(t, \boldsymbol{\alpha})) = \underline{u}(t, \boldsymbol{\alpha}),
        \end{equation}
    is a smooth one-to-one coordinate transformation. Unfortunately, there is not much theory characterizing these maps, which is beyond the scope of this work. But we can say that because $\Phi$ is smooth and invertible, then denoting $\underline{F}$ as the velocity of $\underline{u}$, we have,
        \begin{equation*}
            \begin{split}
            &\argmin_{\boldsymbol{\alpha}} \mathbb{E}_{t \in \mathcal{I}} \left[ \left\| F(\textbf{x}(t, \boldsymbol{\alpha}_0), \boldsymbol{\alpha}_0) - F(\textbf{x}(t, \boldsymbol{\alpha}_0), \boldsymbol{\alpha}) \right\|^2 \right] \\
            &= \argmin_{\boldsymbol{\alpha}} \mathbb{E}_{t \in \mathcal{I}} \left[ \left\| \underline{F}(\underline{u}(t, \boldsymbol{\alpha}_0), \boldsymbol{\alpha}_0) - \underline{F}(\underline{u}(t, \boldsymbol{\alpha}_0), \boldsymbol{\alpha}) \right\|^2 \right]
            \end{split}
        \end{equation*}
    which we prove in the appendix, as Theorem \ref{theorem:time-delay}. So practically, if we find a solution $\boldsymbol{\alpha}$ to the time delayed problem, then we are assured this is also a solution to the original problem with the full state.

\section{Experiments}
\label{sec:experiments}

We examine our method on the following systems: chaotic systems that are affine in the system parameters, the compound double pendulum which is also chaotic but nonlinear in the system parameters, and chaotic real-world data in the Hall-effect thruster (HET).

In all experiments, we use a feed-forward neural network with an input layer, three hidden layers, and an output layer. The hidden layers each have 2,000 nodes, and the activation functions are rectified linear units (ReLUs). Further hyperparameters for each experiment are detailed in Section \ref{sec:hyperparameters}.

\subsection{Chaotic System Affine in System Parameters}
\label{subsec:synthetic_data}

    \begin{table}[ht]
        \centering
        \begin{tabular}{|r|c|l|}
            \hline
            \textbf{Name}    & \textbf{System} & \textbf{Parameter  range} \\ \hline
            Lorenz           & %
            \parbox{6cm}{\begin{equation*}\begin{array}{rl} \dot{x} &= \sigma (y - x) \\ \dot{y} &= x (\rho - z) - y \\ \dot{z} &= xy - \beta z \end{array}\end{equation*}} & %
            \parbox{1cm}{\begin{equation*} \begin{split} \sigma &\in [9, 11] \\ \rho &\in [27, 29] \\ \beta &\in [2, 4] \end{split} \end{equation*}} \\ \hline
            Lorenz96         & %
            \parbox{6cm}{\begin{equation*}\begin{array}{rl} \dot{x}_i &= (x_{i+1} - x_{i-2})x_{i-1} - x_i + F \\ x_0 &= x_{N-1} \\ x_{-1} &= x_N \\  i &=1,\ldots, N \end{array}\end{equation*}} & %
            $F \in [10, 20]$ \\ \hline
            LV Predator-Prey & %
            \parbox{6cm}{\begin{equation*}\begin{array}{rl} \dot{x} &= \alpha x - \beta xy \\ \dot{y} &= \delta xy - \gamma y \end{array}\end{equation*}} & %
            $\alpha, \beta, \gamma, \delta \in [0.5, 1.5]$ \\ \hline
        \end{tabular}
        \caption{List of systems we examine in the experiments that are affine in the system parameters, along with the parameter ranges.}
        \label{table:systems}
    \end{table}

    \begin{table}[ht]
            \centering
            \begin{tabular}{|l|llll|}
            \hline
            \textbf{System (full state) \;\,}    & \textbf{$R^2$ value}      &                                       &                                        &                   \\ \hline
            Lorenz             & \multicolumn{1}{l|}{$\sigma$: 0.99980} & \multicolumn{1}{l|}{$\beta$: 0.99999} & \multicolumn{1}{l|}{$\rho$: 0.99997}   &                   \\ \hline
            Lorenz96           & \multicolumn{1}{l|}{$F$: 0.99999}      &                                       &                                        &                   \\ \hline
            LV Predator-Prey   & \multicolumn{1}{l|}{$\alpha$: 0.99941} & \multicolumn{1}{l|}{$\beta$: 0.99923} & \multicolumn{1}{l|}{$\gamma$: 0.99914} & $\delta$: 0.99885 \\ \hline 
            \end{tabular}\vspace{1em}
            \begin{tabular}{|l|llll|}
            \hline
            \textbf{System (time series)}  & \textbf{$R^2$ value}       &                                       &                                        &                   \\ \hline
            Lorenz           & \multicolumn{1}{l|}{$\sigma$: 0.99972}   & \multicolumn{1}{l|}{$\beta$: 0.99997} & \multicolumn{1}{l|}{$\rho$: 0.99671}   &                   \\ \hline
            Lorenz96         & \multicolumn{1}{l|}{$F$: 0.99999}        &                                       &                                        &                   \\ \hline
            LV Predator-Prey & \multicolumn{1}{l|}{$\alpha$: 0.99956}   & \multicolumn{1}{l|}{$\beta$: 0.99926} & \multicolumn{1}{l|}{$\gamma$: 0.99440} & $\delta$: 0.99860 \\ \hline
            \end{tabular}
        \caption{(Top) The $R^2$ value achieved by our model's prediction for each parameter of each system, when using the original, full state space. As can be observed, the model does really well in predicting the system parameters from trajectories. (Bottom) The $R^2$ value achieved by our model's prediction for each parameter of each system, when using just a time series. As can be observed, the model still does really well in predicting the system parameters from trajectories.}
        \label{table:r2_table}
    \end{table}

    We investigate our method on the systems found in Table \ref{table:systems}, where the training dataset consists of trajectories for parameters in the ranges found in the last column. These trajectories will come from a discrete set of parameters uniformly spaced over the parameter range. For example, for the Lorenz system the trajectories will come from the parameters,
	    \begin{equation*}
	        9 \le \sigma \le 11, \quad 27 \le \rho \le 29, \quad 2 \le \beta \le 4,
	    \end{equation*}
    with a discrete grid spacing of $0.2$ for each parameter, e.g. $(\sigma, \rho, \beta) = (9.2, 28.6, 3.4)$, or $(\sigma, \rho, \beta) = (10.8, 27.0, 3.2)$. For each system, we produce a trajectory to time $t=1,000$, with $100,000$ time-steps (so $\Delta t = 0.01$). The specific details of the training dataset will be found in Section \ref{sec:training-dataset-info}. Our testing dataset consists of 1,000 trajectories for parameters sampled uniformly in the parameter range. We note that for the Lorenz system, within the parameter ranges we have chosen there is a phase transition, where the behaviour moves from chaotic with 2 unstable equilibrium points, to stable with 1 equilibrium point. This demonstrates the ability of our method to deal with such transitions in the data.
	
	We apply our method on two cases: Case 1 is utilizing the original, full state space given in the middle column of Table \ref{table:systems}, and Case 2 is only using the time series given by the first coordinate, but time delaying to reconstruct the state space.

\subsubsection{Using the original state space} 

    \begin{figure}[ht]
        \centering
        \begin{minipage}{0.6\textwidth}
            \begin{framed}
                \centering
                \includegraphics[width=\textwidth]{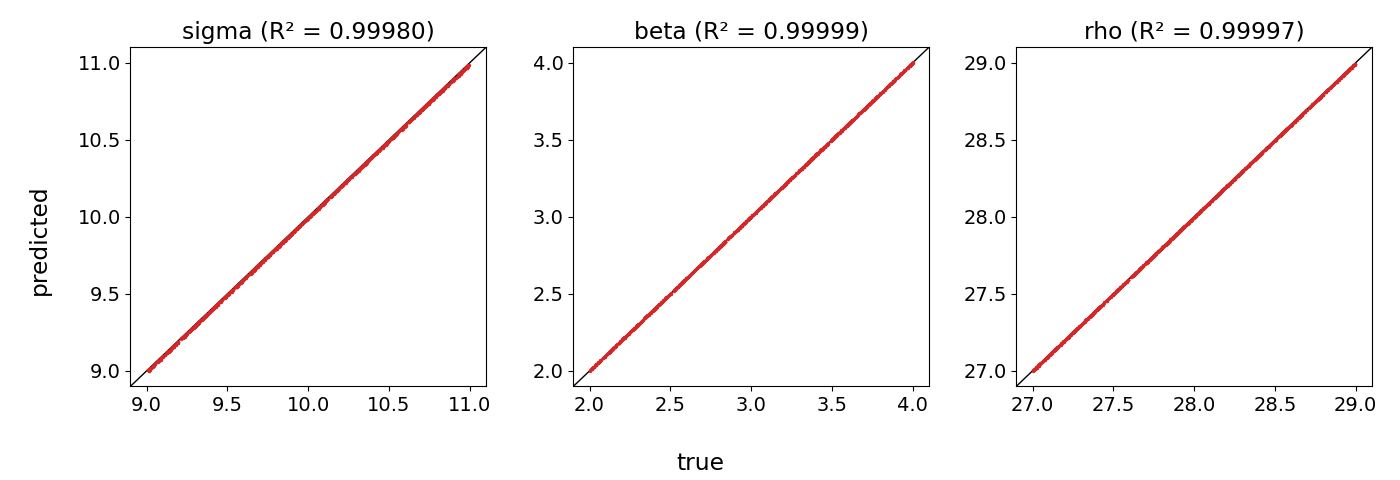}
                \caption*{Lorenz}
            \end{framed}
        \end{minipage}
        \begin{minipage}{0.3\textwidth}
            \begin{framed}
                \centering
                \includegraphics[width=0.74\textwidth]{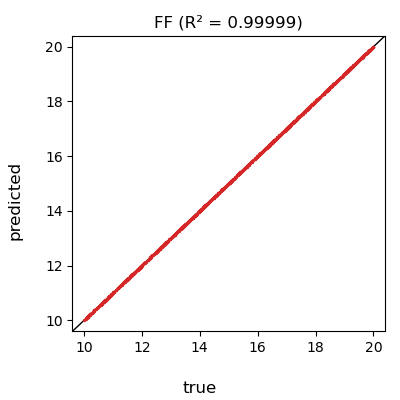}
                \caption*{Lorenz96}
            \end{framed}
        \end{minipage}\vspace{0.5em}
        \begin{minipage}{0.9\textwidth}
            \begin{framed}
                \centering
                \includegraphics[width=\textwidth]{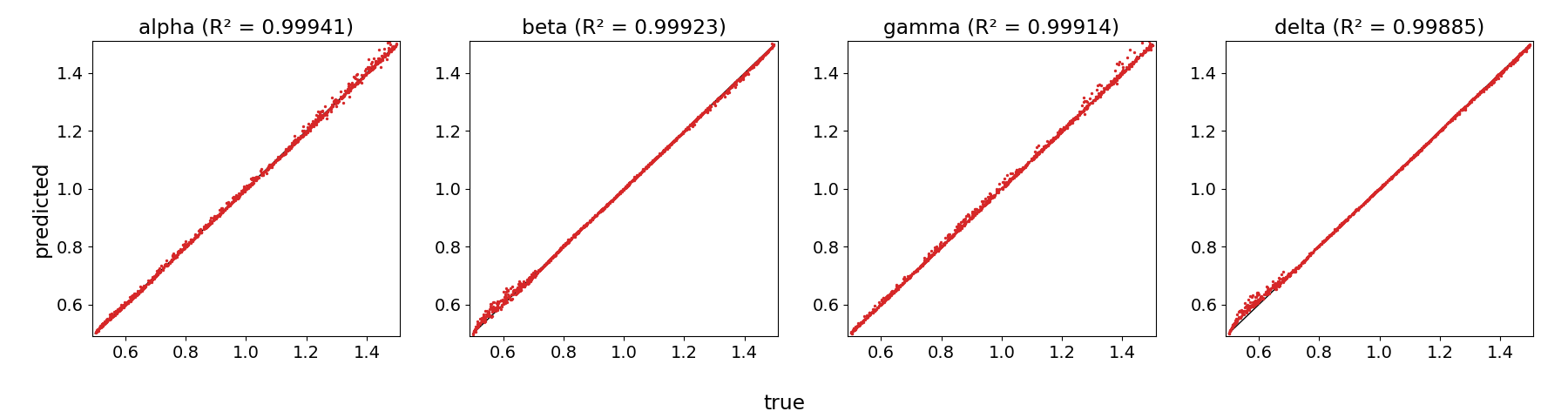}
                \caption*{Lotka-Volterra Predator-Prey}
            \end{framed}
        \end{minipage}
        \caption{The performance of the model's inferred parameters (predicted) vs. the true parameters, when using the original, full, state space. The $R^2$ values of each parameter are also given. As can be observed, the model achieves an $R^2$ value greater than $0.99$ on all parameters for each system.}
        \label{fig:fullstate_r2_plot}
    \end{figure} 
	
	We first examine our method's performance on each dynamical system when it has access to the full, original state space. After learning the velocity operator, we evaluate the model's inference performance by computing the $R^2$ value for each parameter on the test set, which are 1,000 parameters sampled uniformly in the parameter ranges specified in Table \ref{table:systems}. 
 
    The plots of the model's inferred parameters v.s. the true parameters are show in Figure \ref{fig:fullstate_r2_plot} and in Table \ref{table:r2_table} (top) we compute the $R^2$ values achieved by our model for each parameter. As can be observed, the model does pretty well -- on all parameters for each system, the model achieves an $R^2$ value greater than 0.99.

\subsubsection{Using only a time series}

    \begin{figure}[ht]
        \centering
        \begin{minipage}{0.6\textwidth}
            \begin{framed}
                \centering
                \includegraphics[width=\textwidth]{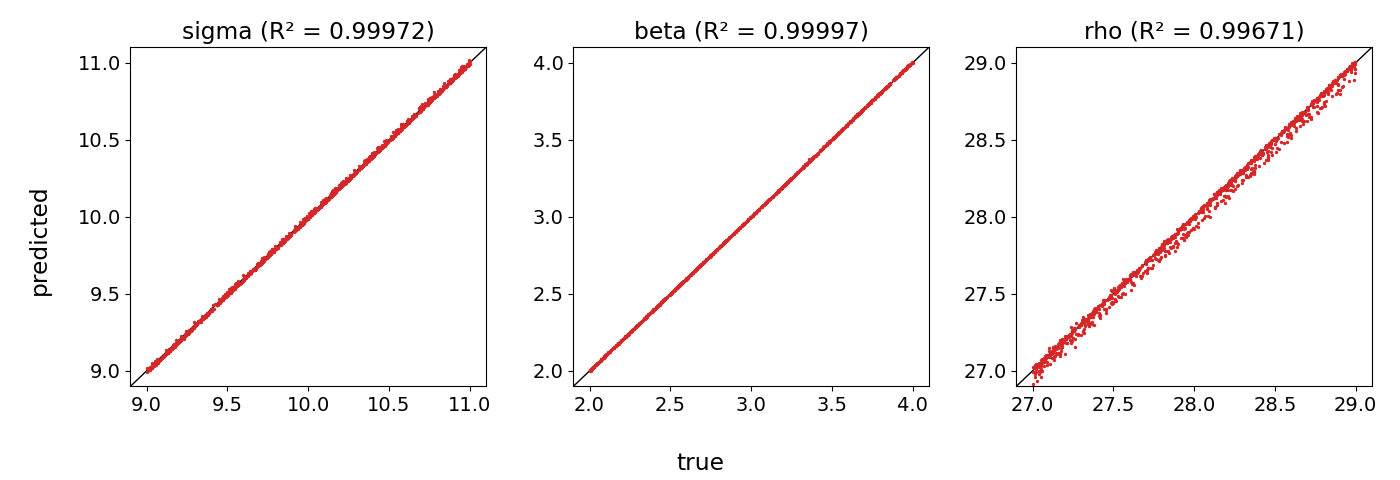}
                \caption*{Lorenz}
            \end{framed}
        \end{minipage}
        \begin{minipage}{0.3\textwidth}
            \begin{framed}
                \centering
                \includegraphics[width=0.74\textwidth]{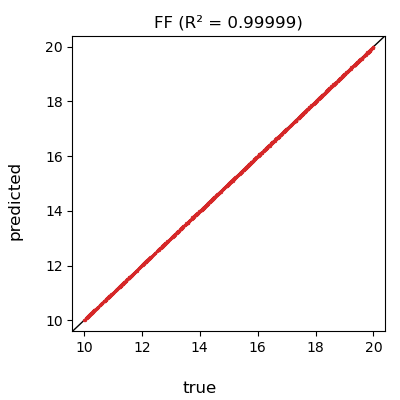}
                \caption*{Lorenz96}
            \end{framed}
        \end{minipage}\vspace{0.5em}
        \begin{minipage}{0.9\textwidth}
            \begin{framed}
                \centering
                \includegraphics[width=\textwidth]{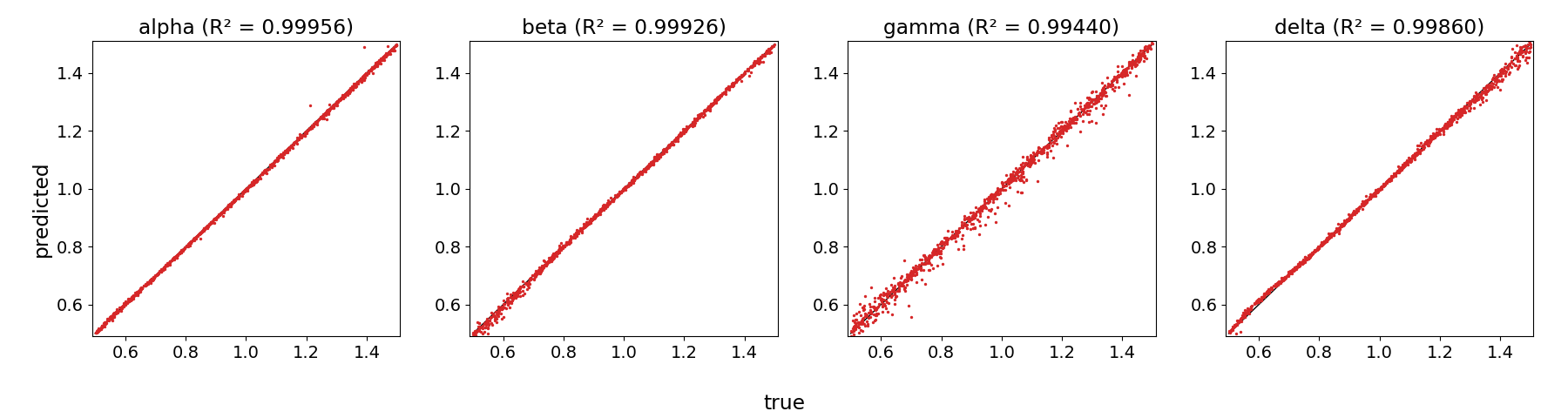}
                \caption*{Lotka-Volterra Predator-Prey}
            \end{framed}
        \end{minipage}
        \caption{The performance of the model's inferred parameters (predicted) vs. the true parameters, when receiving a time delay embedded reconstruction of the state space from a one-dimensional time series. The $R^2$ values of each parameter are also given. While the performance slightly suffers, as can be observed, the model still achieves an $R^2$ value greater than $0.99$ on all parameters for each system.}
        \label{fig:timedelay_r2_plot}
    \end{figure} 

    Now we examine the performance of our method when we just take a one-dimensional time-series from each system and time delay embed it. Methodologically, we first integrating the full system, then we take just the $x$-coordinate (i.e. the first coordinate) and time delay embed this into a higher dimension in order to reconstruct the state space, or more precisely to have a representation of the state space that is diffeomorphic to the original. In Section \ref{sec:extra-figures}, in Figure \ref{fig:timeseries_and_delayembed_lorenz}, we show an example of a trajectory from the Lorenz system, the time series from the $x$-coordinate, and its delay embedding in 3-dimensions. However, in our experiment, we choose a delay embedding dimension of 7-dimensions, in agreement with the guarantees of Takens' theorem.
    
    In Figure \ref{fig:timedelay_r2_plot}, we plot the model's inferred parameters v.s. the true parameters. In Table \ref{table:r2_table} (bottom) we record the $R^2$ value achieved by the model for each parameter for each system. We observe that the performance slightly suffers when using the time delay embedded reconstruction vs. the original, full state space. But overall the performance remains strong, and the statistics demonstrate that our method does pretty well in inferring the true parameters just from \emph{time series} data with parameter labels.

\subsection{The Compound Double Pendulum}
\label{subsec:doublependulum}

    Here we investigate the compound double pendulum, which has the following dynamics:
        \begin{equation*}\left\{\begin{array}{rl} 
            \dot{\theta}_1 &= \frac{6}{m l^2}\frac{2 p_{\theta_1} - 3\cos(\theta_1 - \theta_2) p_{\theta_2}}{16 - 9 \cos^2(\theta_1 - \theta_2)} \\
            \dot{\theta}_2 &= \frac{6}{m l^2}\frac{8 p_{\theta_2} - 3\cos(\theta_1 - \theta_2) p_{\theta_1}}{16 - 9 \cos^2(\theta_1 - \theta_2)} \\
            \dot{p}_{\theta_1} &= -\frac{1}{2} m l^2 \left( \dot{\theta}_1 \dot{\theta}_2 \sin(\theta_1 - \theta_2) + 3 \frac{g}{l} \sin\theta_1 \right) \\
            \dot{p}_{\theta_2} &= -\frac{1}{2} m l^2 \left(-\dot{\theta}_1 \dot{\theta}_2 \sin(\theta_1 - \theta_2) + \frac{g}{l} \sin \theta_2 \right) 
        \end{array}\right.\end{equation*}
    where the system parameters are the mass $m$ and the length $\ell$ -- both pendulums have the same mass and length. In our experiments, we have $m, \ell \in [1, 2]$. The training data consists of trajectories with both mass and length in $[1, 1.1, \ldots, 1.9, 2]$, i.e. the interval $[1, 2]$ with grid-spacing $0.1$. The mapping from system parameters to the dynamics, i.e. $\boldsymbol{\alpha} \mapsto F(\textbf{x}, \boldsymbol{\alpha})$, is \emph{nonlinear}, so that in general the squared error during the inference stage is \emph{nonconvex}. Indeed in Figure \ref{fig:loss-landscape-dubpend-true}, we plot the loss landscape showing its level sets are nonconvex. To ameliorate this, we initialize gradient descent with the system parameter in the training set that achieves the lowest mean squared error. Further hyperparameters and information on the training and test set can be found in Section \ref{subsec:hyperparameters_double_pendulum}.

\subsubsection{Using the original state space}
\label{sec:compound-double-pendulum-full}

    \begin{figure}[ht]
        \centering
        \begin{minipage}{\textwidth}
            \centering
            \includegraphics[width=0.65\textwidth]{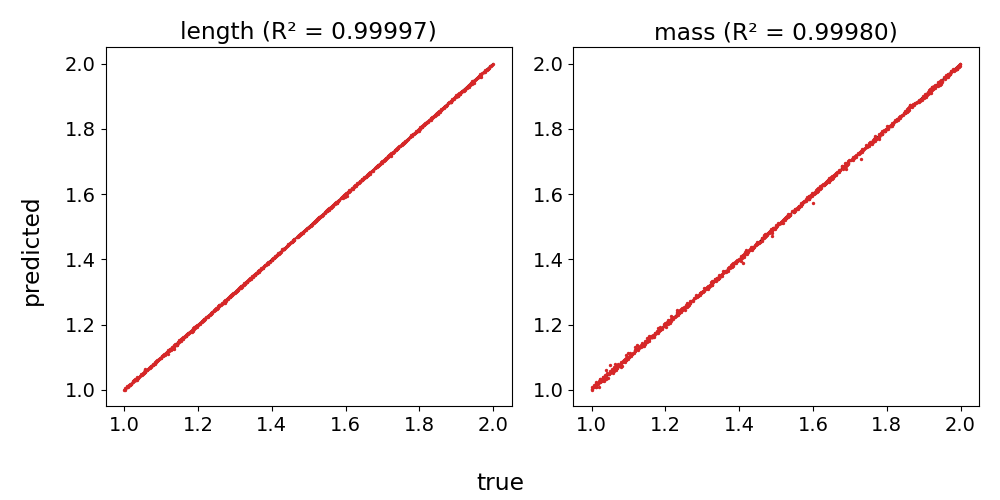}
        \end{minipage}
        \begin{minipage}{0.48\textwidth}
            \includegraphics[width=\textwidth]{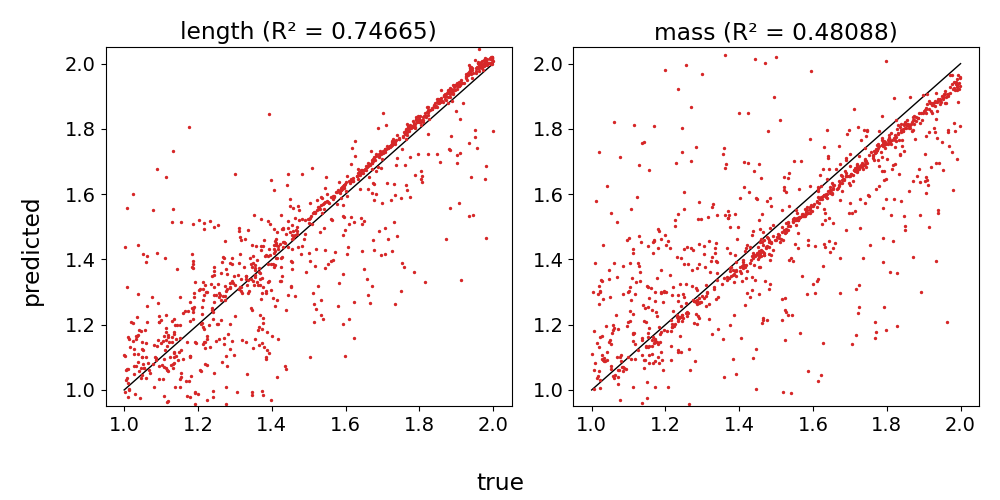}
        \end{minipage}
        \begin{minipage}{0.48\textwidth}
            \centering
            \includegraphics[width=\textwidth]{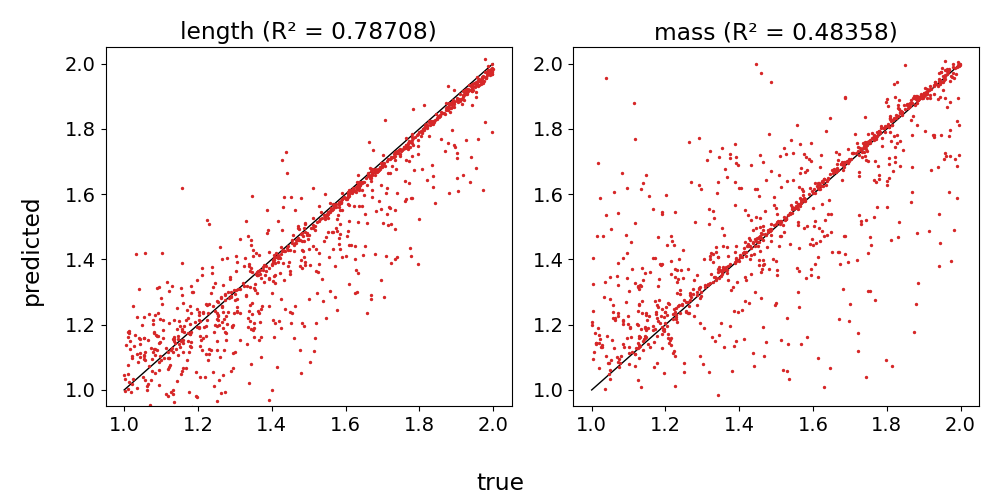}
        \end{minipage}
        \caption{$R^2$ values for determining the mass and length of the compound double pendulum when we have access to the full state. (Top) Here we not only use the full state, but we use spaced out time delays of the full state. As can be seen, we do well inferring the mass and length, with $R^2$ values above $0.999$. (Bottom) Here we do \emph{not} time delay the full state (left), or we perform \emph{consecutive} time delays (right). As can be observed, we do poorly, and surprisingly moreso than when we only have access to the angular time series.}
        \label{fig:double_pendulum_fullstate_r2_plot}
    \end{figure}

    We examine our method when it has access to the full state space. But we have found that just using the full state is not enough, and in order to achieve optimal performance we need to perform a \emph{spaced-out} time delay. The results can be seen in Figure \ref{fig:double_pendulum_fullstate_r2_plot} (Top) where we plot the $R^2$ values of inferring the mass and length of the double pendulum. As can be seen, we do well, with $R^2$ values greater than $0.999$ for both mass and length.

    We make the comparison to when we do \emph{not} use a spaced out time delay of the full state. When we use the naive method of only using the full state without time delays, our performance greatly suffers. The results are shown in Figure \ref{fig:double_pendulum_fullstate_r2_plot} (Bottom-left). If we were to borrow intuition from explicit integration schemes and perform a consecutive time delay (i.e. $\tau=1$), we have found this also performs poorly; the results are plotted in Figure \ref{fig:double_pendulum_fullstate_r2_plot} (Bottom-right). These results imply that using \emph{spaced out} time delays of the full state acts as a regularizer that improves performance. 
    Furthermore, this regularization is in the Learning Phase (Section \ref{subsec:learning_stage}) and not the Inference Stage (Section \ref{subsec:inference_stage}) as all a time delay (spaced or consecutive) does to the Inference loss function is essentially multiply the mean squared error by a factor (the dimension of the embedding). On the other hand, we suspect the effect of the regularization is to smooth the trajectories, as sharp corners are ``smoothed out" when you include other sufficiently far points from the trajectory. If we take the Lorenz system as an example, if you examine the $x$-coordinate time series in Figure \ref{fig:timeseries_and_delayembed_lorenz}, we see there are numerous sharp corners. But if you time delay this $x$-coordinate into 3 dimensions, then because there are more intervals of smoothness than sharp corners, then when we embed the $x$-coordinate times series, the sharp corner becomes ``averaged out".
	We suspect that spaced out time delays of the full state space may also improve performance for other tasks, such as prediction, but this is beyond the scope of this work.

\subsubsection{Using only a time series}
\label{subsubsec:dubpend-time-series}

    Now we examine the performance of our method when we only have access to a one-dimensional time series: the angle of the first/top pendulum. In Figure \ref{fig:double_pendulum_timedelay_r2_plot} we plot the $R^2$ values of mass and length of the double pendulum. As can be observed, we do well in inferring the length of the pendulum, with an $R^2$ value greater than $0.999$, but inferring the mass is a bit harder --  we achieve an $R^2$ value of about $0.939$. Although not as great as the performance for the length, we still do well. 

    \begin{figure}[ht]
        \centering
        \includegraphics[width=0.75\textwidth]{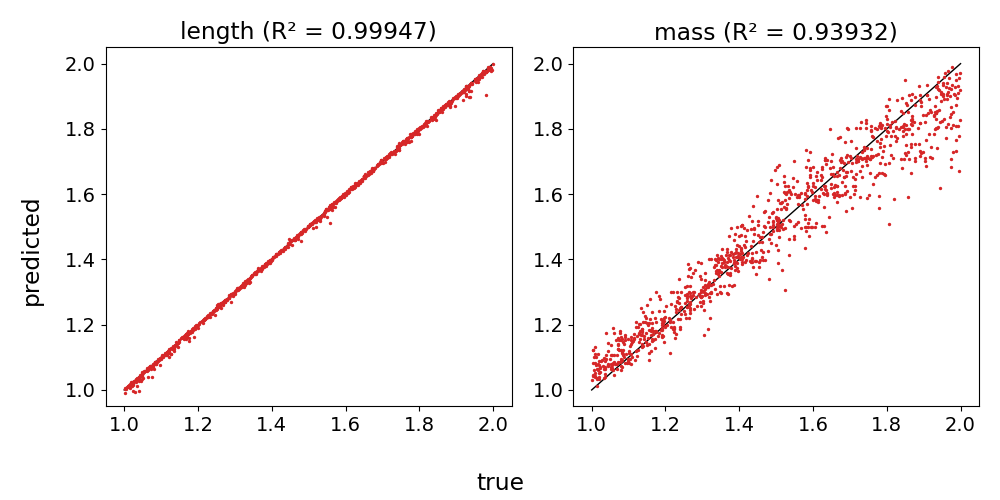}
        \caption{$R^2$ values for determining the mass and length of the compound double pendulum when we only have access to the first angle. We do really well in determining the length, and pretty good in determining the mass.}
        \label{fig:my_label}
        \label{fig:double_pendulum_timedelay_r2_plot}
    \end{figure}

\subsection{Hall-effect Thruster}
    
    In this section, we examine our method on real data: the Hall-effect Thruster (HET), named after the discoverer of the Hall effect, Edwin Hall, which is a type of ion thruster for spacecraft propulsion, where the propellant is accelerated by an electric field \cite{HET_article}. It is well-known, and our data shows, that this system exhibits chaotic behavior.

    \begin{figure}[ht]
        \begin{minipage}{0.49\textwidth}
            \centering
            \includegraphics[width=\textwidth]{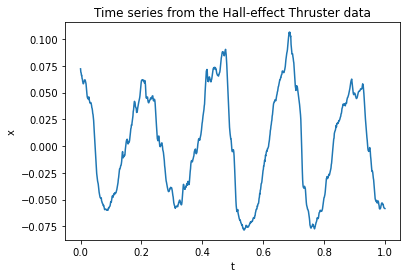}
        \end{minipage} %
        \begin{minipage}{0.49\textwidth}
            \centering
            \includegraphics[width=\textwidth]{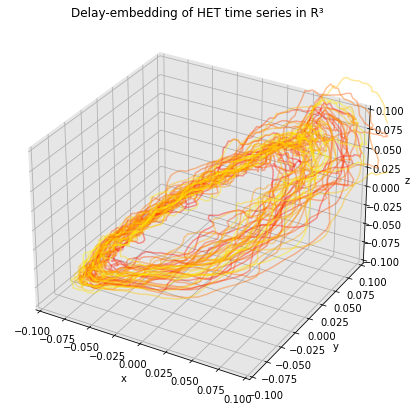}
        \end{minipage}
        \caption{Left: Time series plot for the HET trajectory with parameters $(\alpha, \beta) = (0, 1)$. Right: The delay embedding of the time series on the left into $\mathbb{R}^3$. The time delay is $t=0.01 = 10 \Delta t.$. We note that here we are now dealing with (real) noisy data.}
        \label{fig:timeseries_and_delayembed_het}
    \end{figure}
    
    In our case, the HET data comes as a 2-dimensional list of time series -- namely there are $11\times 31$ time series data, each with one million time-steps. This means we have two parameters which we call $\alpha$ and $\beta$, with $\alpha$ ranging uniformly from about $1$ to $3$ amperes discretized into $11$ values, and $\beta$ ranging uniformly from about $150$ to $500$ volts, discretized into $31$ values. But ultimately this depends on the calibration and specifics of the thruster, and to remain calibration agnostic, we set the specific values of $\alpha$ and $\beta$ to:
        \begin{equation*}
            -1 \le \alpha \le 1, \quad -3 \le \beta \le 3,
        \end{equation*}
    so this means the $11$ values of $\alpha$ will be $-1, -0.8, -0.6, \ldots, 0.8, 1$, and the $31$ values of $\beta$ will be $-3, -2.8, -2.6, \ldots, 2.8, 3$. We also set $\Delta t = 0.001$. In Figure \ref{fig:timeseries_and_delayembed_het} we show an example of a time series, and its delay embedding in $\mathbb{R}^3$. Experimentally, we delay embed the time series in 12-dimensions, with a time delay of $t=0.01$. 

    \begin{figure}[ht]
        \centering
        \includegraphics[width=0.9\textwidth]{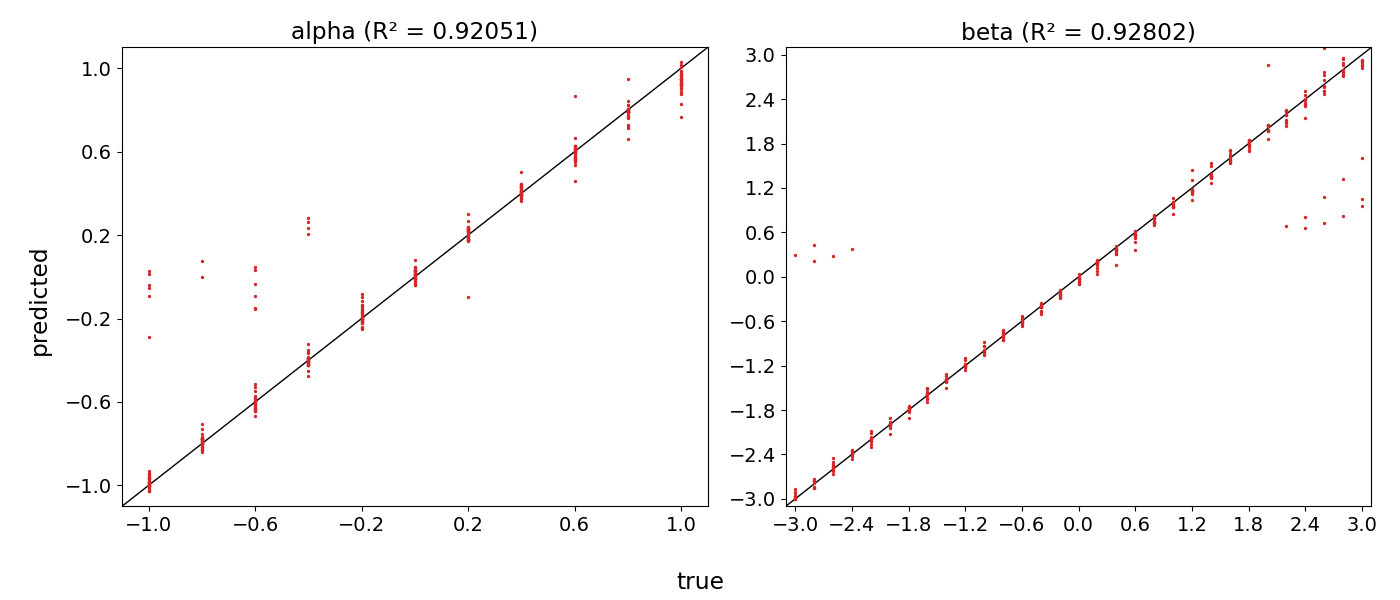}
        \caption{Inferred parameters are plotted against the true parameters, after delay embedding the HET time series. The red dots are the model's inferred parameters vs. the true values. For $\alpha$ we have $R^2 = 0.92051$, and for $\beta$ we have $R^2 = 0.92802$. While not as good as the Lorenz data, taking into account this is real, noisy data, our method seems to perform well. We also note there are outliers here, and of course removing them will improve the $R^2$ values.}
        \label{fig:het_delayembed_r2_plot}
    \end{figure}
    
    For the training data, we exclude the system parameters in a checkerboard fashion as seen in Section \ref{sec:extra-figures}, in Figure \ref{fig:excluded_parameters}. For the testing data we include all parameters. In this way, we are testing our method's ability to infer labels it has never seen before, i.e. to interpolate.

    Using our method on HET trajectories, which contains real-world noise, we compute the $R^2$ values for the inferred parameters v.s. the true parameters in Figure \ref{fig:het_delayembed_r2_plot}. We see that for $\alpha$ we have $R^2 = 0.92051$, and for $\beta$ we have $R^2 = 0.92802$. As expected, we don't do as well compared with the above systems, but considering this is real, noisy data, our method seems to perform well. Clearly, the $R^2$ values are heavily affected by outliers, and when we remove these, of course the $R^2$ value will improve. This demonstrates that not only are we able to do well in inferring the parameters of time series, but also to infer the parameters of time series not in the training data, i.e. it has the ability to interpolate.

\section{Conclusion}

In this work, we present a method to infer the system parameters of dynamical systems given either the full state, or merely a time series. For the scalar time series, we do this by delay embedding in order to reconstruct the state space, and then learn the dynamical system's velocity with a neural network. Furthermore, we can compute gradients of neural networks, allowing us to construct a gradient in parameter space, from which we can infer the system parameters. From our experiments on synthetic and real data, we have demonstrated the efficacy of our method, robustness to partial measurements, and shown it is a promising approach to inferring parameters of dynamical systems.

\section{Acknowledgements}

Alex Tong Lin and Stanley Osher were supported by Air Force Office of Scientific Research (AFOSR) Multidisciplinary University Research Initiative Grant FA9550-18-1-0502, ONR N00014-18-1-2527, N00014-18-20-1-2093, N00014-20-1-2787, Air Fare AWARD 16-EPA-RQ-09. Dan Eckhardt, Robert Martin, and Adrian Wong were supported by AFOSR LRIR FA9550-20RQCOR098 (PO: Fred Leve). We also thank Will Taitano for informative discussions.

\bibliography{references}
\bibliographystyle{plain}

\pagebreak
\appendix

\section{Proof of Theorems}
\label{sec:proof-of-thereoms}

\subsection{Proof of Theorem \ref{theorem:dt-convergence}}
\label{app:dt-convergence-proof}

\begin{proof}
    
    Due to the linearity of the optimization, we are solving a least squares problem. For notational convenience, let $L_{t}:\boldsymbol{\alpha} \mapsto F(\textbf{x}(t, \boldsymbol{\alpha}_0), \boldsymbol{\alpha})$ be the linear mapping of the dynamical system with respect to $\boldsymbol{\alpha}$, and let
        \begin{equation*}
            \textbf{y}_{\Delta t}(t) = \frac{\textbf{x}(t + \Delta t, \boldsymbol{\alpha}_0) - \textbf{x}(t, \boldsymbol{\alpha}_0)}{\Delta t}
        \end{equation*}
    By taking the derivative with respect to $\boldsymbol{\alpha}$ and setting to zero, we obtain,
        \begin{equation*}
            \boldsymbol{\alpha}_{\Delta t} = \mathbb{E}_{t\in\mathcal{I}}[L_{t}^\top L_{t}]^{-1} \mathbb{E}_{t\in\mathcal{I}}[L_{t}^\top \textbf{y}_{\Delta t}(t)]
        \end{equation*}
    The optimization problem fitting to the true velocity and yielding $\boldsymbol{\alpha}_0$ is
        \begin{equation*}
            \begin{split}
            \boldsymbol{\alpha}_0 &= \argmin_{\boldsymbol{\alpha}} \mathbb{E}_{t \in \mathcal{I}} \left[ \left\| F(\textbf{x}(t, \boldsymbol{\alpha}_0), \boldsymbol{\alpha}_0) - F(\textbf{x}(t, \boldsymbol{\alpha}_0), \boldsymbol{\alpha}) \right\|^2 \right] \\
            &= \mathbb{E}_{t\in\mathcal{I}}[L_{t}^\top L_{t}]^{-1} \mathbb{E}_{t\in\mathcal{I}}[L_{t}^\top F(\textbf{x}(t, \boldsymbol{\alpha}_0), \boldsymbol{\alpha}_0)]
            \end{split}
        \end{equation*}
    Now we compute,
        \begin{equation*}
            \begin{split}
            \left\| \boldsymbol{\alpha}_{\Delta t} - \boldsymbol{\alpha}_0 \right\| &= \left\| \mathbb{E}_{t\in\mathcal{I}}[L_{t}^\top L_{t}]^{-1} 
            \mathbb{E}_{t\in\mathcal{I}}[L_{t}^\top \left( \textbf{y}_{\Delta t}(t) - F(\textbf{x}(t, \boldsymbol{\alpha}_0), \boldsymbol{\alpha}_0) \right)] \right\| \\
            &\le B_1 B_2 \; \mathbb{E}_{t\in\mathcal{I}} \left[ \left\| \left( \textbf{y}_{\Delta t}(t) - F(\textbf{x}(t, \boldsymbol{\alpha}_0), \boldsymbol{\alpha}_0) \right) \right\| \right]
            \end{split}
        \end{equation*}
    where $B_1$ is the operator bound for $\mathbb{E}_{t\in\mathcal{I}}[L_{t}^\top L_{t}]^{-1}$ and $B_2$ the operator bound for $L_{t}^\top$. Then by the definition of the derivative, we have that,
        \begin{equation*}
            \lim_{\Delta t\rightarrow 0} \mathbb{E}_{t\in\mathcal{I}} \left[ \left\| \left( \textbf{y}_{\Delta t}(t) - F(\textbf{x}(t, \boldsymbol{\alpha}_0), \boldsymbol{\alpha}_0) \right) \right\| \right] = 0
        \end{equation*}
    and thus $\boldsymbol{\alpha}_{\Delta t} \rightarrow \boldsymbol{\alpha}$ as $\Delta t \rightarrow 0$.
    
\end{proof}

\subsection{Proof of Corollary \ref{corollary:dt-convergence-affine}}

\begin{proof}
    \label{corollary:dt-convergence-affine-proof}
    The optimization problem becomes,
        \begin{equation*}
            \begin{split}
            \boldsymbol{\alpha}_{\Delta t} &= \argmin_{\boldsymbol{\alpha}} \mathbb{E}_{t \in \mathcal{I}} \left[ \left\| \frac{\textbf{x}(t + \Delta t, \boldsymbol{\alpha}_0) - \textbf{x}(t, \boldsymbol{\alpha}_0)}{\Delta t} - F(\textbf{x}(t, \boldsymbol{\alpha}_0), \boldsymbol{\alpha}) \right\|^2 \right] \\
            &= \argmin_{\boldsymbol{\alpha}} \mathbb{E}_{t \in \mathcal{I}} \left[ \left\| \left(\frac{\textbf{x}(t + \Delta t, \boldsymbol{\alpha}_0) - \textbf{x}(t, \boldsymbol{\alpha}_0)}{\Delta t} - b_{\textbf{x}(t, \boldsymbol{\alpha_0})}\right) - L_{\textbf{x}(t, \boldsymbol{\alpha}_0)}\boldsymbol{\alpha} \right\|^2 \right]
            \end{split}
        \end{equation*}
    Setting,
        \begin{equation*}
            \textbf{y}_{\Delta t}(t) = \frac{\textbf{x}(t + \Delta t, \boldsymbol{\alpha}_0) - \textbf{x}(t, \boldsymbol{\alpha}_0)}{\Delta t} - b_{\textbf{x}(t, \boldsymbol{\alpha_0})}
        \end{equation*}
    then we can use the same approach as in the proof of Theorem \ref{theorem:dt-convergence}.
\end{proof}

\subsection{Proof of Theorem \ref{theorem:approx-error-alpha}}

\begin{proof}
\label{theorem:approx-error-alpha-proof}

    Without loss of generality, we assume $F$ is affine. Let
        \begin{equation*}
            \textbf{y}(t) = \frac{\textbf{x}(t + \Delta t, \boldsymbol{\alpha}_0) - \textbf{x}(t, \boldsymbol{\alpha}_0)}{\Delta t}
        \end{equation*}
    If $F$ is affine, we need only modify $\textbf{y}$ by a term that only depends on $\textbf{x}(t, \boldsymbol{\alpha}_0)$, and not $\boldsymbol{\alpha}$. Then we have,
        \begin{equation}\label{eq:y-F-e}
            \mathbb{E}_{t\in\mathcal{I}} \left[ \left\|\textbf{y}(t) - F_\theta(\textbf{x}(t, \boldsymbol{\alpha}_0), \boldsymbol{\alpha}) \right\|^2 \right] = \mathbb{E}_{t\in\mathcal{I}} \left[ \left\|\textbf{y}(t) - F(\textbf{x}(t, \boldsymbol{\alpha}_0), \boldsymbol{\alpha}) - e(\textbf{x}(t, \boldsymbol{\alpha_0}), \boldsymbol{\alpha}) \right\|^2 \right]
        \end{equation}
    Let,
        \begin{equation*}
            L_{t}:\boldsymbol{\alpha} \mapsto F(\textbf{x}(t, \boldsymbol{\alpha}_0), \boldsymbol{\alpha})
        \end{equation*}
    be the linear map. And let,
        \begin{equation*}
            J_t = J_{e(\textbf{x}(t, \boldsymbol{\alpha}_0), \cdot)}
        \end{equation*}
    be the Jacobian of $e$ with respect to $\boldsymbol{\alpha}$ at $e(\textbf{x}(t, \boldsymbol{\alpha}_0), \boldsymbol{\alpha})$. If we compute the derivative of eq. \eqref{eq:y-F-e} with respect to $\boldsymbol{\alpha}$, we get
        \begin{equation}\label{eq:derivative-alpha}
            \begin{split}
            &\mathbb{E}_{t\in\mathcal{I}}\left[L_{t}^\top \left( \textbf{y}(t) - L_{t}\boldsymbol{\alpha} - e(\textbf{x}(t, \boldsymbol{\alpha_0}), \boldsymbol{\alpha}) \right) \right] + \mathbb{E}_{t\in\mathcal{I}}\left[J_{t}( \boldsymbol{\alpha})^\top \left( \textbf{y}(t) - L_{t}\boldsymbol{\alpha} - e(\textbf{x}(t, \boldsymbol{\alpha_0}), \boldsymbol{\alpha}) \right) \right] \\ 
            &= \mathbb{E}_{t\in\mathcal{I}}\left[ L_{t}^\top \textbf{y}(t) \right] - \mathbb{E}_{t\in\mathcal{I}}\left[L_{t}^\top L_{t} \right] \boldsymbol{\alpha} + O(\epsilon)
            \end{split}
        \end{equation}
    where the last equality stems from the assumptions of eq. \ref{eq:sobolev} on $e$, and that $L_{t}$ is a bounded operator. Then we set the above derivative to zero, which is achieved at minimum points $\boldsymbol{\alpha}_{\Delta, \varepsilon}$, and multiply by $\mathbb{E}_{t\in\mathcal{I}}\left[L_{t}^\top L_{t} \right]^{-1}$ to get,
        \begin{equation*}
            \left\|\mathbb{E}_{t\in\mathcal{I}}\left[L_{t}^\top L_{t} \right]^{-1} \mathbb{E}_{t\in\mathcal{I}}\left[ L_{t}^\top \textbf{y}(t) \right] 
 - \boldsymbol{\alpha}_{\Delta t, \varepsilon}\right\| = O(\epsilon)
        \end{equation*}
    But note that
        \begin{equation*}
            \mathbb{E}_{t\in\mathcal{I}}\left[L_{t}^\top L_{t} \right]^{-1} \mathbb{E}_{t\in\mathcal{I}}\left[ L_{t}^\top \textbf{y}(t) \right] 
        \end{equation*}
    is just the solution to the original problem with the true velocity $F$, i.e. $\boldsymbol{\alpha}_{\Delta t}$, so,
        \begin{equation*}
            \left\|\boldsymbol{\alpha}_{\Delta_t} 
 - \boldsymbol{\alpha}_{\Delta t, \varepsilon}\right\| = O(\epsilon).
        \end{equation*}
    Recall that we set the derivative to zero, so $\boldsymbol{\alpha}_{\Delta t, \varepsilon}$ is a minimal point with the approximate velocity $F_\theta$.        

\end{proof}

\subsection{Consistency of the time delay solution}

\begin{lemma}\label{lemma:underlineFsmooth}
    Suppose we are given a dynamical system with full state space dynamics,
        \begin{equation*}
            \frac{d}{dt} \textbf{x}(t, \boldsymbol{\alpha}) = F(\textbf{x}(t, \boldsymbol{\alpha}), \boldsymbol{\alpha}), \quad t\in \mathcal{I} = [0, T]
        \end{equation*}
    but we only observe a time series $u(t)$ which can be delay embedded with embedding $\Phi$, as in eq. \eqref{eq:time-series_delay} and eq. \eqref{eq:delay-reconstruction-map}. Then if $F$ is smooth with respect to the space variable $\textbf{x}$, then the time delay velocity map
        \begin{equation*}
            \underline{F}(\underline{u}(t, \boldsymbol{\alpha}), \boldsymbol{\alpha}) = \frac{d}{dt} \underline{u}(t, \boldsymbol{\alpha})
        \end{equation*}
    is also smooth with respect to $\underline{u}$.
\end{lemma}

\begin{proof}
    If we let $V$ be the flow generated by the vector field $F$, so that
        \begin{equation*}
            V(\textbf{x}_0, \boldsymbol{\alpha}, t) = \textbf{x}(t, \boldsymbol{\alpha}), \quad \textbf{x}(0, \boldsymbol{\alpha}) = \textbf{x}_0
        \end{equation*}
    then we have $\frac{d}{dt} V(\textbf{x}_0, \boldsymbol{\alpha}, t) = F(\textbf{x}(t, \boldsymbol{\alpha}), \boldsymbol{\alpha})$. Note that since the vector field $F$ is smooth, then the flow $V$ is also smooth. We can also observe, and also following \cite{CASDAGLI199152}, that if $\underline{V}$ is the flow for $\underline{F}$, we have the following relationship
        \begin{equation*}
            \underline{V}(\underline{u}_0, \boldsymbol{\alpha}, t) = \Phi \circ V(\cdot, \boldsymbol{\alpha}, t) \circ \Phi^{-1}(\underline{u}_0)
        \end{equation*}
    then,
        \begin{equation*}
            \begin{split}
            \underline{F}(\underline{u}_0, \boldsymbol{\alpha}, t) = \frac{d}{dt} \underline{V}(\underline{u}_0, \boldsymbol{\alpha}, t) &= \frac{d}{dt} \Phi \circ V(\cdot, \boldsymbol{\alpha}, t) \circ \Phi^{-1}(\underline{u}_0) \\ 
            \end{split}
        \end{equation*}
    and thus the flow field $\underline{F}$ is also smooth, being the composition of smooth functions (where $\circ$ is composition).
\end{proof}

\begin{theorem}\label{theorem:time-delay}
    Suppose we have a dynamical system with full state space dynamics,
        \begin{equation*}
            \frac{d}{dt} \textbf{x}(t, \boldsymbol{\alpha}) = F(\textbf{x}(t, \boldsymbol{\alpha}), \boldsymbol{\alpha}), \quad t\in \mathcal{I} = [0, T]
        \end{equation*}
    with $F$ smooth with respect to the space variable $\textbf{x}$, but we only observe a time series $u(t)$ which can be delay embedded with embedding $\Phi$, as in eq. \eqref{eq:time-series_delay} and eq. \eqref{eq:delay-reconstruction-map}. Then
        \begin{equation*}
            \boldsymbol{\alpha}^* = \argmin_{\boldsymbol{\alpha}} \mathbb{E}_{t \in \mathcal{I}} \left[ \left\| F(\textbf{x}(t, \boldsymbol{\alpha}_0), \boldsymbol{\alpha}_0) - F(\textbf{x}(t, \boldsymbol{\alpha}_0), \boldsymbol{\alpha}) \right\|^2 \right]
        \end{equation*}
    if and only if
        \begin{equation*}
            \boldsymbol{\alpha}^* = \argmin_{\boldsymbol{\alpha}} \mathbb{E}_{t \in \mathcal{I}} \left[ \left\| \underline{F}(\underline{u}(t, \boldsymbol{\alpha}_0), \boldsymbol{\alpha}_0) - \underline{F}(\underline{u}(t, \boldsymbol{\alpha}_0), \boldsymbol{\alpha}) \right\|^2 \right]
        \end{equation*}
\end{theorem}

\begin{proof}
    Suppose
        \begin{equation*}
            \boldsymbol{\alpha}^* = \argmin_{\boldsymbol{\alpha}} \mathbb{E}_{t \in \mathcal{I}} \left[ \left\| \underline{F}(\underline{u}(t, \boldsymbol{\alpha}_0), \boldsymbol{\alpha}_0) - \underline{F}(\underline{u}(t, \boldsymbol{\alpha}_0), \boldsymbol{\alpha}) \right\|^2 \right]
        \end{equation*}
    We note that because $\boldsymbol{\alpha}_0$ is a solution, then this means,
        \begin{equation}\label{eq:equals-zero}
            \mathbb{E}_{t \in \mathcal{I}} \left[ \left\| \underline{F}(\underline{u}(t, \boldsymbol{\alpha}_0), \boldsymbol{\alpha}_0) - \underline{F}(\underline{u}(t, \boldsymbol{\alpha}_0), \boldsymbol{\alpha}^*) \right\|^2 \right] = 0
        \end{equation}
    From Lemma \ref{lemma:underlineFsmooth}, we know $\underline{F}$ is smooth, and therefore from the above equation, we have must have
        \begin{equation}\label{eq:same_velocity}
            \underline{F}(\underline{u}(t, \boldsymbol{\alpha}_0), \boldsymbol{\alpha}_0) = \underline{F}(\underline{u}(t, \boldsymbol{\alpha}_0), \boldsymbol{\alpha}^*), \quad t\in \mathcal{I}
        \end{equation}
    But because $\underline{F}$ is smooth, then setting, 
        \begin{equation}\label{eq:underline_u_inits}
            \underline{u}(0, \boldsymbol{\alpha}^*) = \underline{u}(0, \boldsymbol{\alpha}_0)    
        \end{equation} 
    we get
        \begin{equation*}
            \underline{F}(\underline{u}(0, \boldsymbol{\alpha}_0), \boldsymbol{\alpha}_0) = \underline{F}(\underline{u}(0, \boldsymbol{\alpha}^*), \boldsymbol{\alpha}_0) = \underline{F}(\underline{u}(0, \boldsymbol{\alpha}^*), \boldsymbol{\alpha}^*)
        \end{equation*}
    where the first equality is from eq. \eqref{eq:underline_u_inits}, and the second equality is from eq. \eqref{eq:same_velocity}. But then due to the equality of velocity maps in eq. \eqref{eq:same_velocity}, and because the trajectories have the same point at time $t=0$,
        \begin{equation}\label{eq:u-equality-t}
            \underline{u}(t, \boldsymbol{\alpha}_0) = \underline{u}(t, \boldsymbol{\alpha}^*), \quad t \in \mathcal{I}.
        \end{equation}
    which also means,
        \begin{equation}\label{eq:F-equality-t}
            \underline{F}(\underline{u}(t, \boldsymbol{\alpha}_0), \boldsymbol{\alpha}_0) = \underline{F}(\underline{u}(t, \boldsymbol{\alpha}^*), \boldsymbol{\alpha}^*), \quad t\in\mathcal{I}.
        \end{equation}
    
    Now, from the proof of Lemma \ref{lemma:underlineFsmooth}, we know that the flow fields have the following relationship,
        \begin{equation*}
            \underline{V}(\underline{u}_0, \boldsymbol{\alpha}, t) = \Phi \circ V(\Phi^{-1}(\underline{u}_0), \boldsymbol{\alpha}, t) = \Phi \circ V(\textbf{x}_0, \boldsymbol{\alpha}, t)
        \end{equation*}
    And we note,
        \begin{equation*}
            \underline{F}(\underline{u}(t, \boldsymbol{\alpha}), \boldsymbol{\alpha}) = \frac{d}{dt} \underline{V}(\underline{u}_0, \boldsymbol{\alpha}, t)
        \end{equation*}
    Then using the chain rule, we have for all $\boldsymbol{\alpha}$,
        \begin{equation}\label{eq:F-chain-rule}
            \begin{split}
                \underline{F}(\underline{u}(t, \boldsymbol{\alpha}), \boldsymbol{\alpha}) &= J_\Phi(\Phi^{-1} (\underline{u}(t, \boldsymbol{\alpha}))) F(\Phi^{-1}(\underline{u}(t, \boldsymbol{\alpha})), \boldsymbol{\alpha}) \\
                &= J_\Phi(\textbf{x}(t, \boldsymbol{\alpha})) F(\textbf{x}(t, \boldsymbol{\alpha}), \boldsymbol{\alpha})
            \end{split}
        \end{equation}
    and therefore using eq. \eqref{eq:u-equality-t}, eq. \eqref{eq:F-equality-t}, and eq. \eqref{eq:F-chain-rule}, then eq. \eqref{eq:equals-zero} becomes,
        \begin{equation*}
            \begin{split}
            0 &= \mathbb{E}_{t \in \mathcal{I}} \left[ \left\| \underline{F}(\underline{u}(t, \boldsymbol{\alpha}_0), \boldsymbol{\alpha}_0) - \underline{F}(\underline{u}(t, \boldsymbol{\alpha}_0), \boldsymbol{\alpha}^*) \right\|^2 \right] \\
            &= \mathbb{E}_{t \in \mathcal{I}} \left[ \left\| J_\Phi(\textbf{x}(t, \boldsymbol{\alpha}_0)) F(\textbf{x}(t, \boldsymbol{\alpha}_0), \boldsymbol{\alpha}_0) - \underline{F}(\underline{u}(t, \boldsymbol{\alpha}^*), \boldsymbol{\alpha}^*) \right\|^2 \right] \\
            &= \mathbb{E}_{t \in \mathcal{I}} \left[ \left\| J_\Phi(\textbf{x}(t, \boldsymbol{\alpha}_0)) F(\textbf{x}(t, \boldsymbol{\alpha}_0), \boldsymbol{\alpha}_0) - J_\Phi(\textbf{x}(t, \boldsymbol{\alpha}^*)) F(\textbf{x}(t, \boldsymbol{\alpha}^*), \boldsymbol{\alpha}^*) \right\|^2 \right] \\
            &= \mathbb{E}_{t \in \mathcal{I}} \left[ \left\| J_\Phi(\textbf{x}(t, \boldsymbol{\alpha}_0)) F(\textbf{x}(t, \boldsymbol{\alpha}_0), \boldsymbol{\alpha}_0) - J_\Phi(\textbf{x}(t, \boldsymbol{\alpha}_0))F(\textbf{x}(t, \boldsymbol{\alpha}_0), \boldsymbol{\alpha}^*) \right\|^2 \right]
            \end{split}
        \end{equation*}
    where we used,
        \begin{equation*}
            \textbf{x}(t, \boldsymbol{\alpha}^*) = \Phi^{-1}(\underline{u}(t, \boldsymbol{\alpha}^*)) = \Phi^{-1}(\underline{u}(t, \boldsymbol{\alpha}_0)) = \textbf{x}(t, \boldsymbol{\alpha}_0)
        \end{equation*}
    By smoothness, this means,
        \begin{equation*}
            \begin{split}
            0 &= \left\| J_\Phi(\textbf{x}(t, \boldsymbol{\alpha}_0)) F(\textbf{x}(t, \boldsymbol{\alpha}_0), \boldsymbol{\alpha}_0) - J_\Phi(\textbf{x}(t, \boldsymbol{\alpha}_0))F(\textbf{x}(t, \boldsymbol{\alpha}_0), \boldsymbol{\alpha}^*) \right\|^2 \\
            &= \left\| J_\Phi(\textbf{x}(t, \boldsymbol{\alpha}_0)) \left(F(\textbf{x}(t, \boldsymbol{\alpha}_0), \boldsymbol{\alpha}_0) - F(\textbf{x}(t, \boldsymbol{\alpha}_0), \boldsymbol{\alpha}^*) \right) \right\|^2
            \end{split}
        \end{equation*}
    But note that since $\Phi$ is invertible, then $J_\Phi \neq 0$. Then we must have 
        \begin{equation*}
            F(\textbf{x}(t, \boldsymbol{\alpha}_0), \boldsymbol{\alpha}_0) = F(\textbf{x}(t, \boldsymbol{\alpha}_0), \boldsymbol{\alpha}^*)
        \end{equation*}
    so then,
        \begin{equation*}
            0 = \mathbb{E}_{t \in \mathcal{I}} \left[ \left\| F(\textbf{x}(t, \boldsymbol{\alpha}_0), \boldsymbol{\alpha}_0) - F(\textbf{x}(t, \boldsymbol{\alpha}_0), \boldsymbol{\alpha}^*) \right\|^2 \right]
        \end{equation*}
    and thus,
        \begin{equation*}
            \boldsymbol{\alpha}^* = \argmin_{\boldsymbol{\alpha}} \mathbb{E}_{t \in \mathcal{I}} \left[ \left\| F(\textbf{x}(t, \boldsymbol{\alpha}_0), \boldsymbol{\alpha}_0) - F(\textbf{x}(t, \boldsymbol{\alpha}_0), \boldsymbol{\alpha}) \right\|^2 \right]
        \end{equation*}
    A similar argument proves the other direction.

\end{proof}

\section{Linear and Affine With Respect to System Parameters}
\label{sec:linear-and-affine}

\subsection{Linearity of Lotka-Volterra Predator Prey}

The Lotka-Volterra Predator Prey system is,
    \begin{equation*}
        \left.\begin{array}{rl} 
        \dot{x} &= \alpha x - \beta xy \\ 
        \dot{y} &= \delta xy - \gamma y 
        \end{array}\right\} = 
        \left[\begin{matrix}
            x & -xy &  0 &  0 \\
            0 &   0 & xy & -y
        \end{matrix}\right] 
        \left[\begin{matrix}
            \alpha \\
            \beta \\
            \delta \\
            \gamma
        \end{matrix}\right]
    \end{equation*}

\subsection{Affineness of Lorenz}

The Lorenz System is,
    \begin{equation*}
        \left.\begin{array}{rl} 
            \dot{x} &= \sigma (y - x) \\ 
            \dot{y} &= x (\rho - z) - y \\ 
            \dot{z} &= xy - \beta z 
        \end{array}\right\} = 
        \left[\begin{matrix}
            y - x &  0 & 0 \\
            0       &  x & 0 \\
            0       &  0 & -z
        \end{matrix}\right] 
        \left[\begin{matrix}
            \sigma \\
            \rho \\
            \beta
        \end{matrix}\right] +
        \left[\begin{matrix}
            0 \\
            -xz - y \\
            xy
        \end{matrix}\right]
    \end{equation*}

\section{Training Dataset Info}

\label{sec:training-dataset-info}

For all training datasets except the compound double pendulum, we integrated each system over uniformly-spaced time points of $dt=0.01$, from $t=0$ to $t=1,000$ (so $10^5 + 1$ time points). For the compound double pendulum, we had $dt=0.0001$ from $t=0$ to $t=100$ (so $10^6 + 1$ time points). For integration, we used the Python package ``SciPy", namely the function \texttt{scipy.integrate.solve\_ivp} method, with the integration method set to ``Radau" (the Radau integration method).   

\subsection{Lorenz}

For the Lorenz system, we used system parameters in the range:
    \begin{equation*}
        9 \le \sigma \le 11, \quad 2 \le \beta \le 4, \quad 27 \le \rho \le 29,
    \end{equation*}
    with a spacing of $0.2$ for each parameter. So some example parameters that were in the training dataset are: $(9.0, 2.2, 28.4)$, or $(10.2, 1.8, 27.8)$, but NOT $(9.1, 2.2, 28.4)$ because of the $9.1$.

    The initial point for all trajectories is $x_0 = (0, 1, 1.05)$.

    When we time delay the Lorenz system, we used a time delay of $t=1.6$ ($=dt \times 16 = 0.01 \times 16$), and embedding dimension 7.

\subsection{Lorenz96}

For the Lorenz96 system, we used system parameters in the range:
    \begin{equation*}
        10 \le F \le 20,
    \end{equation*}
    with a spacing of $0.2$. So some example parameters that were in teh training dataset are: $F=10.2$, or $F=15.2$, but NOT $F=16.1$. We used a 4 dimensional  Lorenz96 system.

    The initial point for all trajectories is $x_0 = (-2.46820633, 0.09570264, 1.59270902, 10.21372147)$

    When we time delay the Lorenz96 system, we used a time delay of $0.2$ with an embedding dimension of 9.

\subsection{Lotka-Volterra Predator-Prey (LVPP)}

For the Lotka-Volterra Predator-Prey (LVPP), we used system parameters in the range:
    \begin{equation*}
        \alpha, \beta, \gamma, \delta \in [0.5, 1.5]
    \end{equation*}
    with spacing of $0.2$. So some example parameters that were in the training dataset are: $(1.2, 0.8, 0.6, 0.6)$, or $(1.4, 1.8, 0.8, 0.2)$, but NOT $(1.1, 1.8, 0.8, 0.2)$ because of the $1.1$.

    The initial point for all trajectories is $x_0 = (3, 3)$.

    When we time delay the Lorenz system, we used a time delay of $t=0.1$ ($=dt \times 16 = 0.01 \times 16$), and embedding dimension 5.

\subsection{Compound Double Pendulum}

For the Compound Double Pendulum, we used system parameters in the range:
    \begin{equation*}
        m, \ell \in [1, 2]
    \end{equation*}
    with spacing of $0.1$. So during training both $m$ and $\ell$ take on values in the set $\{1.0, 1.1, \ldots, 1.9, 2.0\}$.

    The initial point for all trajectories was $x_0 = (-44.334542, 223.53554, -1.2249799, 2.535486)$.

    When we time delay the full state to achieve the optimal performance in Figure \ref{fig:double_pendulum_fullstate_r2_plot} (Top), we use a time delay of $t=0.1$ ($=dt \times 1,000 = 0.0001 \times 1,000$), with an embedding dimension of $12 = 4 \times 3$. So we used two time delays plus the original state. Namely if $\textbf{x}(t_k)$ is the original state, then we used $(\textbf{x}(t_k), \textbf{x}(t_{k - m k}), \textbf{x}(t_{k - 2m k}))$, with $m = 1,000$.

    When we time delay the double pendulum when we only had the time series of the angle as analyzed in Section \ref{subsubsec:dubpend-time-series}, we used a time delay of $t=0.1$ ($=dt \times 1,000 = 0.0001 \times 1,000$). 

\section{Hyperparameters}
\label{sec:hyperparameters}

In all our experiments, we use a feed-forward neural network with an input layer, 3 hidden layers, and an output layer. The 3 hidden layers have 2,000 nodes each. And the activation function is the rectified linear unit (ReLU).

\subsection{Chaotic System Affine in System Parameters}
\label{subsec:hyperparameters_lorenz}

For the learning phase, the training hyperparameters for our neural network are:

\begin{itemize}
    \item Learning rate: $10^{-4}$,
    \item Batch size: $500$
    \item Epochs: $2\times 10^6$
\end{itemize}

For the inference stage, we use stochastic gradient descent with momentum, with the following hyperparameters:

\begin{itemize}
    \item Learning rate: $10^{-4}$,
    \item Batch size: $500$,
    \item Momentum: $0.99$
    \item Maximum iterations: $20,000$.
\end{itemize}

For Lorenz and LVPP, the initial starting point of the gradient descent is the middle of the parameter ranges. For Lorenz96, we chose the initial starting point to be the parameter in the training set that minimized the $\ell^2$ error.

\subsection{Compound Double Pendulum}
\label{subsec:hyperparameters_double_pendulum}

For the learning phase, the training hyperparameters for our neural network are:

\begin{itemize}
    \item Learning rate: $10^{-3}$
    \item Batch size: $500$
    \item Epochs: $2\times 10^6$
\end{itemize}

For the inference stage, when we have access to the full state, we use the Adam optimizer, with the following hyperparameters,

\begin{itemize}
    \item Learning rate: $10^{-4}$
    \item Batch size: $1000$
    \item Beta values (for Adam): $\beta_1 = 0.9$, $\beta_2 = 0.999$
    \item Maximum iterations: $10,000$
\end{itemize}

When we only have access to the first angle's time series, then the learning rate is $10^{-5}$.

\subsection{Hall-effect Thruster}
\label{subsec:hyperparameters_het}

For the learning phase, the training hyperparameters for our neural network are:

\begin{itemize}
    \item Learning rate: $10^{-5}$,
    \item Batch size: $200$,
    \item Max epochs: $2\times 10^{6}$.
\end{itemize}

For the inference stage, we use stochastic gradient descent with momentum, with the following hyperparameters:

\begin{itemize}
    \item Learning rate: $10^{-1}$,
    \item Batch size: $200$,
    \item Momentum: $0.5$,
    \item Maximum iterations; $20,000$.
\end{itemize}

\section{Extra Figures}
\label{sec:extra-figures}

\begin{figure}[ht]
    \begin{minipage}{\textwidth}
        \centering
        \includegraphics[width=0.5\textwidth]{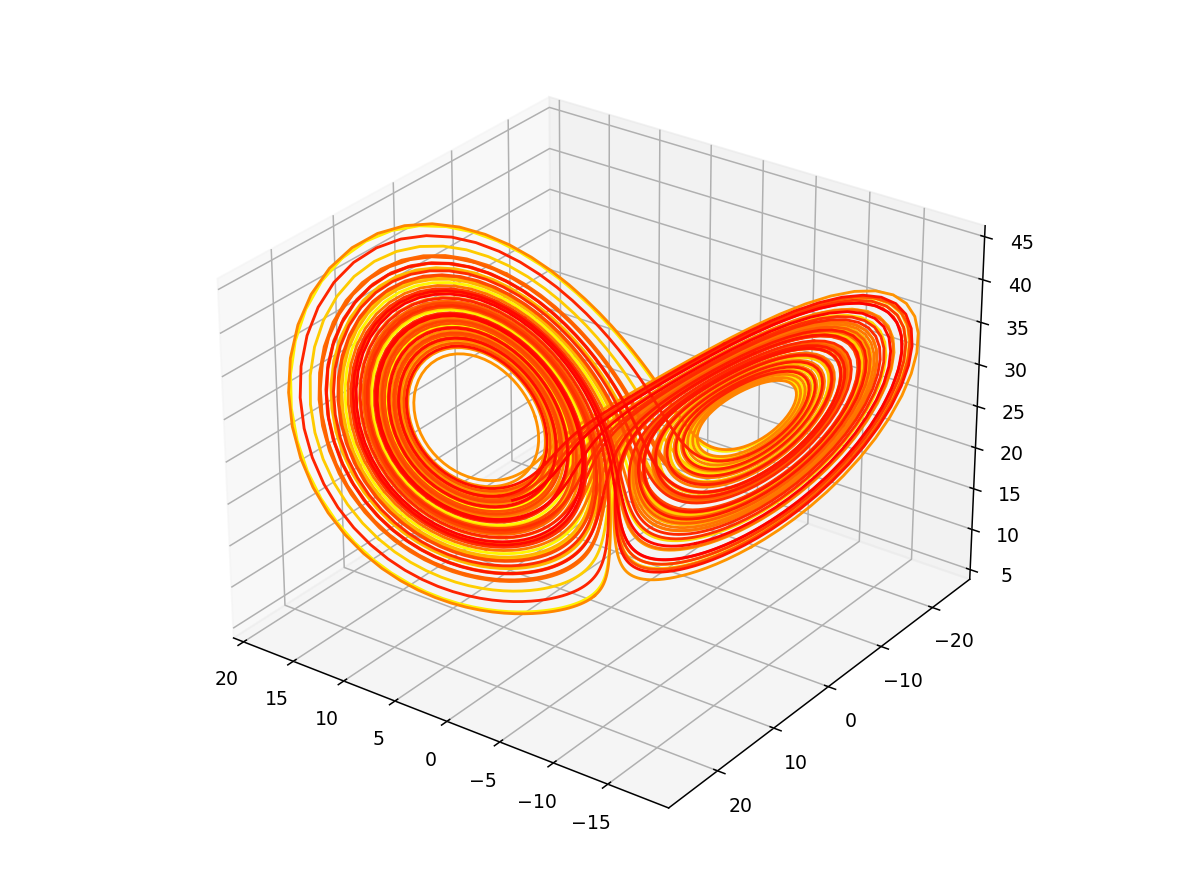}
    \end{minipage}
    \begin{minipage}{0.49\textwidth}
        \centering
        \includegraphics[width=\textwidth]{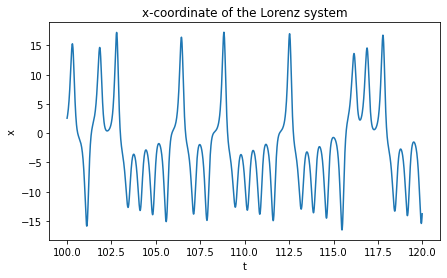}
    \end{minipage} %
    \begin{minipage}{0.49\textwidth}
        \centering
        \includegraphics[width=\textwidth]{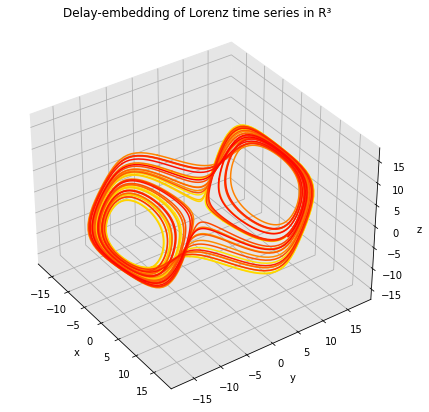}
    \end{minipage}
    \caption{(Top) Canonical example trajectory of the Lorenz system, with $\sigma=10$, $\beta=8/3$, and $\rho=28$. (Left) Time series plot for the trajectory with parameters $(\sigma, \beta, \rho) = (10, 8/3, 28)$. (Right) The delay embedding of the time series on the left into $\mathbb{R}^3$. The time delay is $t=0.16 = 16 \Delta t.$.}
    \label{fig:timeseries_and_delayembed_lorenz}
\end{figure}

\begin{figure}[ht]
    \centering
    \includegraphics[width=0.65\textwidth]{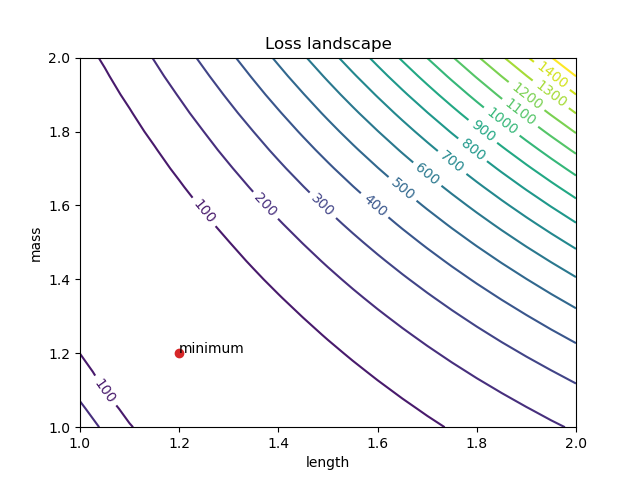}
    \caption{We plot the loss landscape, as a contour plot, of the optimization problem $\min_{\boldsymbol{\alpha}} \mathbb{E}_{t\in\mathcal{I}}\left[F(\textbf{x}(t, \boldsymbol{\alpha}_0), \boldsymbol{\alpha})  - F(\textbf{x}(t, \boldsymbol{\alpha}_0), \boldsymbol{\alpha}_0)\right]$, where $F$ is the right-hand side of the compound double pendulum ODE, and the length is $1.2$ and the mass is $1.2$, i.e. $\boldsymbol{\alpha}_0 = (1.2, 1.2)$. As can be seen, the level sets are nonconvex, thus the optimization problem is also nonconvex.}
    \label{fig:loss-landscape-dubpend-true}
\end{figure}

\begin{figure}[ht]
    \centering
    \includegraphics[width=0.9\textwidth]{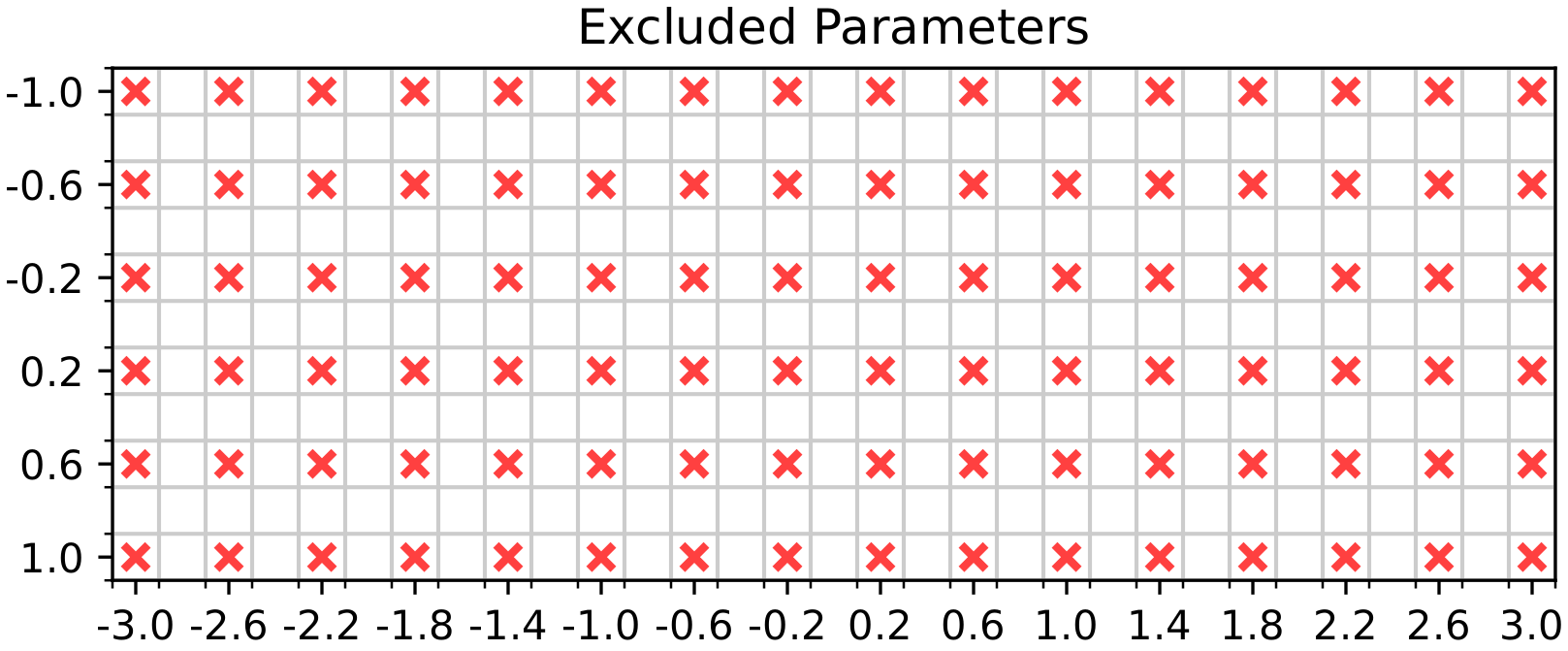}
    \caption{Visualization of the parameters that were excluded from the training data for the HET experiment. The parameters with red x markers are excluded from the training data, e.g. $(\alpha, \beta) = (-1, -3)$, or $(\alpha, \beta)=(0.2, 0.6)$.}
    \label{fig:excluded_parameters}
\end{figure}

\end{document}